\documentclass{article}

    \PassOptionsToPackage{numbers, compress}{natbib}

    \usepackage[preprint]{neurips_2024}

\usepackage{subcaption}
\usepackage[utf8]{inputenc} 
\usepackage[T1]{fontenc}    
\usepackage[colorlinks=true]{hyperref}       
\hypersetup{
    citecolor=blue
}
\usepackage{url}            
\usepackage{booktabs}       
\usepackage{amsfonts}       
\usepackage{amssymb,amsmath,amsthm}
\usepackage{microtype}      
\usepackage{algorithm}
\usepackage{algpseudocode}  
\usepackage{graphicx}
\usepackage{thmtools}
\usepackage{caption} 

\usepackage{lipsum}         
\usepackage{xspace}         
\usepackage{xargs}          
\usepackage{mathtools}
\usepackage{bbm}
\usepackage{bm}
\usepackage[dvipsnames]{xcolor} 
\usepackage{tabularx}       
\usepackage{wrapfig}        
\usepackage[
]{cryptocode}

\usepackage[export]{adjustbox}

\usepackage{tikz}
\usepackage{ifthen}
\usepackage{enumitem}
\usepackage{cleveref}
\usetikzlibrary{patterns}
\usetikzlibrary{positioning}

\newtheorem{thm}{Theorem}

\DeclarePairedDelimiterX{\inp}[2]{\langle}{\rangle}{#1, #2}
\DeclarePairedDelimiterX{\norm}[1]{\lVert}{\rVert}{#1}
\DeclarePairedDelimiterX{\cbr}[1]{\{}{\}}{#1} 
\DeclarePairedDelimiterX{\rbr}[1]{(}{)}{#1} 
\DeclarePairedDelimiterX{\sbr}[1]{[}{]}{#1} 

\providecommand{\R}{\mathbb{R}} 

\DeclareMathOperator{\expect}{\mathbb{E}}
\DeclareMathOperator{\E}{\expect}

\DeclareMathOperator{\sgn}{sign}
\makeatletter
\def\sign{\@ifnextchar*{\@sgnargscaled}{\@ifnextchar[{\sgnargscaleas}{\@ifnextchar{\bgroup}{\@sgnarg}{\sgn} }}}
\def\@sgnarg#1{\sgn\rbr{#1}}
\def\@sgnargscaled#1{\sgn\rbr*{#1}}
\def\@sgnargscaleas[#1]#2{\sgn\rbr[#1]{#2}}
\makeatother

\DeclareMathOperator*{\argmax}{arg\,max}

\providecommand{\0}{\bm{0}}

\renewcommand{\gg}{\bm{g}}

\providecommand{\xx}{\bm{x}}
\providecommand{\yy}{\bm{y}}
\providecommand{\zz}{\bm{z}}

\newcommand{\bmu}{\boldsymbol{\mu}}
\newcommand{\muv}{\bmu}

\providecommand{\mW}{\bm{W}}
\providecommand{\mX}{\bm{X}}

\providecommand{\cC}{\mathcal{C}}
\providecommand{\cD}{\mathcal{D}}

\providecommand{\cO}{\mathcal{O}}

\providecommand{\cT}{\mathcal{T}}

\newtheorem{theorem}{Theorem}

\newtheorem{corollary}[theorem]{Corollary}

\newtheorem{lemma}{Lemma}

\newtheorem{remark}[lemma]{Remark}
\newtheorem{example}[lemma]{Example}

\newtheoremstyle{shortassumptionstyle}%
{}
{}
{\itshape}
{0.5em}
{\bfseries}
{.}
{0.5em}
{(\thmname{#1}\thmnumber{#2}) \thmnote{#3}}
\theoremstyle{shortassumptionstyle}
\newtheorem{assumption}{A}

\crefname{assumption}{}{}
\Crefname{assumption}{}{}
\creflabelformat{assumption}{(#2A#1#3)}
\theoremstyle{plain}

\newcommand{\ignore}[1]{}

\definecolor{color1}{RGB}{228,26,28}
\definecolor{color2}{RGB}{55,126,184}
\definecolor{color3}{RGB}{77,175,74}
\definecolor{color4}{RGB}{152,78,163}
\definecolor{color5}{RGB}{255,127,0}

\usepackage{pifont}
\makeatletter
\newcommand{\myitem}[1]{%
    \item[\textbf{(#1)}]\protected@edef\@currentlabel{#1}%
}
\makeatother

\usetikzlibrary{fit,calc}

\colorlet{worker}{red!40}
\newcommand{\speedup}[1]{{\color{gray}(\ifdim #1 pt > 0.3pt #1\else $< #1$\fi{}$\times$)}}

\renewcommand{\epsilon}{\varepsilon}

\newcommand\rsmraise[1]{%
  \ifx#1\displaystyle .8\else
    \ifx#1\textstyle .8\else
      \ifx#1\scriptstyle .6\else
        .45%
      \fi
    \fi
  \fi}
\definecolor{steelblue}{rgb}{0.27, 0.51, 0.71}
\definecolor{firebrick}{rgb}{0.7, 0.13, 0.13}
 
\DeclareMathAlphabet{\mathpzc}{OT1}{pzc}{m}{it}

\newcommand{\ouralgo}{\textsc{CoBo}\xspace}

\title{
	\ouralgo: Collaborative Learning via Bilevel Optimization 
}

\author{
  Diba Hashemi \smallskip \\
  EPFL \\
  \texttt{diba.hashemi@epfl.ch} \\
  \And
  Lie He \smallskip \\
  EPFL \\
  \texttt{lie.he@epfl.ch} \\
  \And
  Martin Jaggi \\
  EPFL \\
  \texttt{martin.jaggi@epfl.ch} \\
}

\begin{document}

\maketitle

\begin{abstract}

Collaborative learning is an important tool to train multiple clients more effectively by enabling communication among clients. Identifying helpful clients, however, presents challenging and often introduces significant overhead. In this paper, we model \textit{client-selection} and \textit{model-training} as two interconnected optimization problems, proposing a novel bilevel optimization problem for collaborative learning.
We introduce \ouralgo, a \textit{scalable} and \textit{elastic}, SGD-type alternating optimization algorithm  that efficiently addresses these problem with theoretical convergence guarantees. Empirically, \ouralgo achieves superior performance, surpassing popular personalization algorithms by 9.3\% in accuracy on a task with high heterogeneity, involving datasets distributed among 80 clients.\footnote{The code is available at: https://github.com/epfml/CoBo}
\end{abstract}

\section{Introduction}
In a classic collaborative learning scenario, $n$ clients, each with a distinct machine learning task, seek solutions that potentially outperform their individual solvers through a collective effort.
Common collaborative learning frameworks generally alternate between training local models on individual datasets and synchronizing updates between collaborators.
More concretely, during the computation step, client $i \in [n]$ trains a $d$-dimensional model $\xx_i \in \R^d$ to minimize its loss function, $f_i: \R^d \rightarrow \R$. In the subsequent communication step, client $i$ exchanges updates with collaborators, potentially benefiting from collaboration.

While there is a plethora of collaborative learning frameworks, the ideal way to collaborate remains under-exploited. The \textsc{FedAvg}~\citep{mcmahan2017communication,kairouz2021advances} algorithm learns one global model over pooled datasets from all clients, i.e., $\min_{\xx \in \mathbb{R}^d} \frac{1}{n}\sum_{i=1}^{n} f_i(\xx)$. However, due to heterogeneous data distributions between clients, a global model may significantly under-perform compared to personal models trained on local datasets for certain clients, which can discourage their participation in collaborative training \citep{mohri2019agnostic}.
 \textsc{Ditto} trains personal models with a regularization term that penalizes their deviation from a global model \citep{smith2021ditto}. Although \textsc{Ditto} enables personal models to leverage the global model, it offers only a coarse-grained level of collaboration. 
In instances where clients' data exhibit significant differences, the \textsc{Ditto} algorithm is constrained to facilitating collaboration at a global level, thereby neglecting the inherent client heterogeneity structure.

Clustering-based federated learning algorithms have been developed to accommodate scenarios in which clients' data originate from multiple clusters \citep{ghosh2020efficient,werner2023provably}. Nevertheless, these algorithms typically inherit the limitations associated with clustering techniques, including the need to predetermine the number of clusters, initialize cluster centers, and other such prerequisites, which can diminish their practical utility.

In this paper, we propose to use a bilevel optimization framework to enhance collaborative learning, by discovering better structural relationships among clients. 
The inner problem focuses on optimizing a binary collaborator selection variable $w_{ij}\in\{0,1\}$, which is determined based on a gradient alignment measure for each pair of clients. In the outer problem, we train personalized models $\mX\in\R^{n\times d}$, while incorporating a penalization term that accounts for the distances between clients, as dictated by the collaboration weights established in the inner problem. 

The contributions of the paper can be summarized as follows.
\begin{itemize}[left=0.3cm]
    \item We model collaborative learning through a novel bilevel optimization formulation that yields more generalizable solutions by fully exploiting the inherent structure of collaboration.
    \item We propose \ouralgo, an SGD-type alternating optimization algorithm that efficiently solves the bilevel problem. \ouralgo scales with the number of clients $n$ and is elastic to the number of clients.
    \item \ouralgo is proved to enjoy theoretical convergence guarantees for collaborative learning with cluster structures.
    \item Empirically, \ouralgo surpasses popular personalized federated learning baselines in experiments involving highly heterogeneous federated learning settings and Large Language Models (LLMs).
\end{itemize}

\section{Problem formulation}
\label{sec:prob-formul}

In this paper, we model collaborative learning as a bilevel optimization problem, where personalized models $\mX \in \R^{d \times n}$ are trained in the outer problem, and collaborative weights $\mW \in \R^{n \times n}$ are given by the inner problem. More concretely,
\begin{align}
    \label{eq:outer} \tag{Model-Training}
    &\min_{[\xx_1, \dots, \xx_n]\in\R^{d\times n}}  \sum\limits_{i = 1}^n f_i(\xx_i)  + \frac{\rho }{2}\sum\limits_{1 \leq i < j \leq n} w_{ij}^\star \norm*{\xx_i - \xx_j}^2_2  \\
  \text{where } w_{ij}^\star \in& \argmax_{w_{ij} \in [0, 1]}~w_{ij} \left\langle \nabla f_i\left(\frac {\xx_i + \xx_j}{2}\right), \nabla f_j\left(\frac{\xx_i + \xx_j}{2}\right) \right\rangle \quad \forall ~i,j\in[n] \label{eq:inner}\tag{Client-Selection}
\end{align}
where $\rho > 0$ is a hyper-parameter for penalization. We break down the formulation as follows.

\paragraph{Outer problem: training personalized models.}
In the outer problem~\eqref{eq:outer}, client $i$ trains its model $\xx_i$ by minimizing its loss function $f_i$, along with its distances to the neighbor models, e.g. $\xx_j$, as defined by the weight $w_{ij}^\star > 0$. 

Our formulation is similar to \textsc{Ditto}~\cite{smith2021ditto}, but with two key differences: \textsc{Ditto} uses uniform and fixed collaboration weight and penalizes the distance between~$\xx_i$ and a global model, whereas we penalize the distance between pairs of clients and adjust the collaboration weight during training. Consequently, when client tasks are heterogeneous, such as clients are drawn from clusters, the performance of a global model deteriorates and \textsc{Ditto}'s local model cannot benefit from fine-grained collaboration. Our method, on the other hand, is able to exploit such structure and achieve better performance in diverse settings.

\begin{figure}[t]
    \centering
    \includegraphics[width=0.9\textwidth]{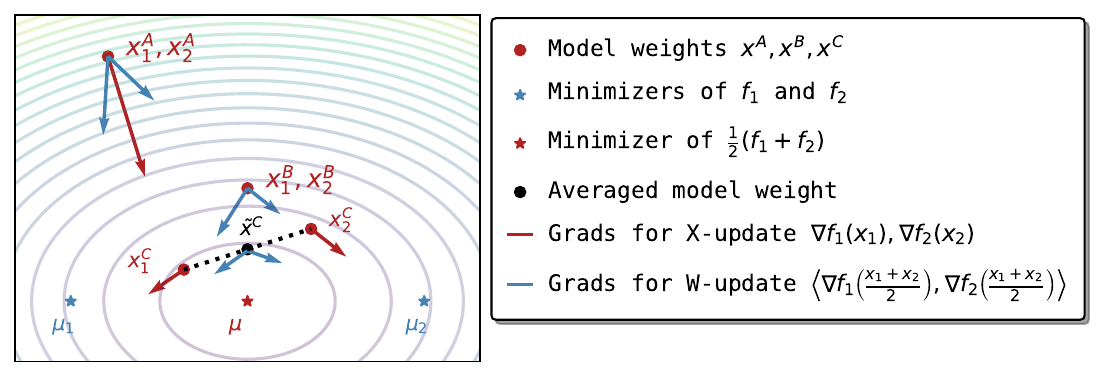}
    \caption{
        Diagram of the inner problem \eqref{eq:inner} represented through a contour of $\frac{1}{2}(f_1+f_2)$. The blue arrows {\color{steelblue}$\rightarrow$} are gradients computed at middle point $\frac{1}{2}(\xx_1+\xx_2)$ to determine connectivity. The red arrows {\color{firebrick}$\rightarrow$} represent gradients computed at local models to update model weights.
    }
    \label{fig:diagram}
\end{figure}

\paragraph{Inner Problem: Finding Collaborators}
In the inner problem, we decompose the task of optimizing~$\mW \in \R^{n \times n}$ into independent sub-problems, one for each entry of $\mW$. The binary collaborator selection $w_{ij} \in \{0,1\}$ is relaxed to a continuous weight $w_{ij} \in [0,1]$. As the objective function is linear with respect to $w_{ij}$, and the domain is convex, solvers such as Frank-Wolfe \citep{frank1956algorithm,jaggi2013revisiting} or projected gradient descent can efficiently find the maximizers at $0$ or $1$. 

It is important to note that $w_{ij}^\star$ does not imply a permanent connection between clients $i$ and $j$, but rather a temporary assessment based on the current states of $\xx_i$ and $\xx_j$.

A simple inner problem with two clients is illustrated in \Cref{fig:diagram}. The $f_1$, $f_2$ are their loss functions. $\muv_1$, $\muv_2$, and $\muv$ are the minimizers of $f_1$, $f_2$, and $\frac{1}{2} (f_1 + f_2)$. Suppose  $\muv_1$, $\muv_2$, and $\muv$ are minimizers of $f_1$, $f_2$, and $\frac{1}{2}(f_1+f_2)$ respectively. The model weights at $A,B,C$ demonstrates three scenarios to update $\mW$.
\begin{itemize}[itemsep=5pt, left=0.2cm]
    \item \textbf{Point A:} model $\xx^A$ is far away from $\muv$, i.e., $\norm{\xx^A-\muv} >> \max_i \norm{\muv_i-\muv}$. Descent directions of clients have positive inner product and therefore $w_{12}=1$. Collaboration at this stage speeds up training.

    \item \textbf{Point B:} model $\xx^B$ is closer to $\muv$, i.e., $\norm{\xx^B-\muv} \sim \norm{\muv_i-\muv}$. In this case, moving closer to the minimizer $\muv$ of $\frac{1}{2}(f_1+f_2)$ no longer help both clients to get closer to minimizers of their own losses $\muv_i$. The inner problem yields $w_{12}=0$ and disconnect clients.
    
    \item \textbf{Point C:} models $\xx_1^C$ and $\xx_2^C$ are already disconnected. The gradients computed at their midpoint suggest they should remain disconnected, i.e., $w_{12}=0$.
\end{itemize}

 Collaboration weights in \ref{eq:inner} are determined in a pair-wise fashion. In contrast to clustering-based methods~\citep{ghosh2020efficient,werner2023provably}, this formulation does not require knowledge of cluster sizes, allowing clients to join and leave during collaborative training. Our formulation enabling \textit{elasticity} and applies to more scenarios.

\begin{remark}[Extensions]
    While \ref{eq:inner} is defined over a box constraint $\mW \in [0,1]^{n \times n}$, it can be easily extended to other convex domains. For example, in all-for-one type collaborative training, the domain weight is optimized over a simplex,. The experiment on language models is deferred to \Cref{subsec:lang-exp}.
\end{remark}

\subsection{Algorithm}

We propose a novel SGD-type alternating-optimization algorithm, termed \ouralgo, to solve the bilevel optimization problem defined by \eqref{eq:outer} and \eqref{eq:inner}. The algorithm alternates between updating the model variables $\mX$ and the collaboration weights $\mW$.

In each round $t$, the model variables $\{\xx_i^t\}_{i=1}^n$ are first fixed, and the collaboration weights $\mW$ are updated by applying projected gradient descent with step size $\gamma > 0$ to \eqref{eq:inner}:
\begin{equation} \label{eq:w}
    w^{t+1}_{ij} =\text{Proj}_{[0,1]} \left(w^t_{ij} + \gamma \left\langle \nabla f_i \left(\frac{\xx_i^t+\xx_j^t}{2}\right), \nabla f_j \left(\frac{\xx_i^t+\xx_j^t}{2}\right) \right\rangle \right) \qquad \forall i,j \in [n].
\end{equation}
Next, the updated collaboration weights $\{w_{ij}^{t+1}\}$ are fixed, and the model variables $\{\xx_i\}_{i=1}^n$ are optimized using the following update rule for \eqref{eq:outer}:
\begin{align} \label{eq:x}
    \xx_i^{t+1} = \xx_i^t - \eta \left(\nabla f_i(\xx_i^t) + \rho \sum_{k=1}^n w_{ik}^{t+1} \left(\xx_i^t - \xx_k^t\right) \right) \qquad \forall i \in [n],
\end{align}
where $\eta > 0$ is the step size. This alternating process is repeated until convergence.

The detailed implementation is described in Algorithm \ref{alg:update-params}. In this implementation, the full gradients $\{\nabla f_i\}_{i \in [n]}$ in \eqref{eq:w} and \eqref{eq:x} are replaced by their stochastic estimates. Additionally, collaborative weights are updated with a probability of $\cO(\frac{1}{n})$, resulting in an expected computation of $\cO(n)$ gradients. This incurs an overhead similar to standard decentralized learning \citep{lian2017decentralized,koloskova2021unified}, effectively enabling client selection with minimal additional cost.

Compared to federated clustering algorithms, which require global synchronization before applying clustering oracles, \eqref{eq:inner} in \ouralgo is carried out in pairs. Such pair-wise operation makes the algorithm non-blocking and robust to stragglers, providing greater flexibility and efficiency.

\begin{algorithm}[t]
\caption{\ouralgo: \textbf{Co}llaborative Learning via \textbf{B}ilevel \textbf{O}ptimization }\label{alg:update-params}
\textbf{Input: }
Model parameters $\forall~i\in[n]~\xx_i^0=\xx^0\in\R^d$; 
Penalization parameter $\rho>0$; $\mW^0 \in \R^{n\times n}$ where $w_{ij}^0=1, \forall~i,j\in[n] $;
Step size $\eta,\gamma>0$. 
\begin{algorithmic}[1]
\For{round $t=0,1\ldots, T$}
    \State Call $\mW^{t+1} \leftarrow \texttt{Client-Selection}(\{\xx_i^{t}\}_{i\in[n]}, \mW^t)$
    \For{client $i=1,\ldots n$}
        \State Draw sample $\xi_i\sim\cD_i$ and compute stochastic gradient $\gg_i^t\in\R^d$ of $f_i(\xx_i^t)$ and update
        \begin{align} \label{eq:x_}
            \xx_i^{t+1} \leftarrow \xx_i^t - \eta \left( \gg_i^t + \rho \sum_{k=1}^n w_{ik}^{t+1} \left(\xx_i^t - \xx_k^t\right) \right)
        \end{align}
    \EndFor
\EndFor
\State{\bfseries Output:} Uniform randomly select $s\in[T]$ and return $\{\xx_0^s, \dots, \xx_n^s\}$ and $\mW^{s}$. \\

\Procedure{Client-Selection}{$\mX$, $\mW$} 
    \For{each pair of clients $(i, j)$ where $i\neq j \in [n]$}
        \If{with a probability $1/n$, }
            \State Compute the average model $\zz_{ij} = \frac{1}{2}(\xx_i + \xx_j)$.
            \State Compute stochastic gradient $\gg_{i\leftarrow i}$ and $\gg_{i\leftarrow j}$ for $f_i(\zz_{ij})$ and $f_j(\zz_{ij})$ respectively,
            \begin{equation} \label{eq:w_}
                w_{ij} \leftarrow \text{Proj}_{[0,1]} \left(w_{ij} + \gamma \left\langle \gg_{i\leftarrow i}, \gg_{i\leftarrow j}\right\rangle \right).
            \end{equation}    
        \EndIf
    \EndFor
\State{\bfseries return} updated selection variables $\mW$
\EndProcedure

\end{algorithmic}
\end{algorithm}

\section{Theoretical results}\label{sec:theoretical-results}
In this section, we define assumptions in collaborative learning settings and show that \ouralgo converges to stationary points. The following assumptions regarding the local optimization objective~$f_i$ are commonly adopted in the literature~\citep{arjevani2023lower,koloskova2021unified}:
\begin{assumption}[$L$-smooth]\label{a:smoothness}
For all $\xx$ and $\yy$ in $\mathbb R^d$ and $i\in[n]$, the loss function $f_i$ has $L$-Lipschitz gradients, i.e. 
\[
  \norm*{\nabla f_i(\xx) - \nabla f_i(\yy)} \leq L \norm{\xx - \yy} \, .
\]
\end{assumption}

\begin{assumption}[Noise bound]\label{a:noise-bound}
For all $\xx\in\mathbb R^d$ and $i\in[n]$, there exists $\sigma^2 > 0$ such that the stochastic gradient has bounded noise
\[
  \E_\xi \left[\norm*{\nabla f_i(\xx; \xi) - \E_\xi \left[ \nabla f_i(\xx; \xi) \right]}^2 \right] \leq \sigma^2  \, .
\]
\end{assumption}

\begin{assumption}[Global minimum]\label{a:global_minimum}
For all $i\in[n]$, the loss function $f_i$ has a global lower bound $f_i^\ast$. 
\end{assumption}

The next assumption characterizes the possible relationships between clients. In the first case, when reaching the stationary point $\xx$ of their joint objective $f_i + f_j$, then by \eqref{eq:in_cluster} implies that $\nabla f_i(\xx) = \nabla f_j(\xx) = \0$ client $i$ and $j$ reach their own stationary points. In the second case, when client $i$ reaches its stationary point, the gradient of $j$ is lower bounded by a positive constant, meaning they don't share stationary points. This leads to eventual 
\begin{assumption}[Collaborativeness]\label{a:collaborative}
    If clients $i$ and $j$ are collaborative, then there exists $M_{ij}>0$ such that
    \begin{equation}\label{eq:in_cluster}
        \norm*{\nabla f_i(\xx) - \nabla f_j(\xx)}_2^2 \le M_{ij}^2\norm*{\nabla f_i(\xx) + \nabla f_j(\xx)}_2^2 \qquad \forall~\xx\in\R^d.
    \end{equation}
    Otherwise, there exists $\zeta^2_{ij}>0$ such that
    \begin{equation}\label{eq:outside_cluster}
        \norm*{\nabla f_i(\xx)}_2^2 + \norm{\nabla f_j(\xx)}_2^2 \ge\zeta^2_{ij} \qquad \forall~\xx\in\R^d.
    \end{equation}
\end{assumption}

This assumption is similar to \citep[Assumptions 4,5]{werner2023provably}, but we define relations for pairs of clients instead of clusters.
In the next example, we use quadratics to demonstrate \Cref{a:collaborative} 
\begin{example}
    Assume that there are $K$ clusters with $[n]=\cup_{k\in[K]}~\cC_k$ and $\cC_k\cap\cC_{k'}=\emptyset$ for all $k\neq k'\in[K]$. Consider the $k$-th cluster with center $\muv_k$ and client $i\in\cC_k$, the loss function is $f_i(\xx) = \frac{a_i}{2} \norm*{\xx-\muv_k}_2^2$ where $a_i>0$. Then for clients $i,j$ in the same cluster, i.e. $i,j\in\cC_k$
\begin{align*}
    \norm*{\nabla f_i(\xx) - \nabla f_j(\xx)}_2^2=(a_i-a_j)^2 \norm*{\xx - \muv_k}_2^2
    =\frac{(a_i-a_j)^2}{(a_i+a_j)^2}\norm*{\nabla f_i(\xx) + \nabla f_j(\xx)}_2^2.
\end{align*}
The $M_{ij}=\frac{|a_i-a_j|}{a_i+a_j}$ in this case. On the other hand, for $i\in\cC_k$ and $j\in\cC_{k'}$ and $\muv_k\neq \muv_{k'}$,
\begin{align*}
    \norm*{\nabla f_i(\xx)}_2^2 + \norm{\nabla f_j(\xx)}_2^2=a_i^2\norm*{\xx - \muv_k}_2^2 + a_j^2\norm*{\xx - \muv_{k'}}_2^2
    = \frac{a_i^2a_j^2}{(a_i^2 + a_j^2)^2} \norm*{\muv_k - \muv_{k'}}_2^2
\end{align*}
where the lower bound $\zeta_{ij}^2=\frac{a_i^2a_j^2}{(a_i^2 + a_j^2)^2} \norm*{\muv_k - \muv_{k'}}_2^2>0$.
\end{example}

Finally, we derive a convergence theorem with the assumption that clients are drawn from clusters, as e.g. in \citep[Assumption 2]{sattler2019}. 
\begin{assumption}[Cluster]\label{a:cluster} 
    All clients are drawn from clusters where within each cluster clients share stationary points.
\end{assumption}

\begin{theorem}\label{theorem:main}
    Suppose Assumption \ref{a:smoothness},\ref{a:noise-bound},\ref{a:global_minimum},\ref{a:collaborative},\ref{a:cluster} hold true.
    Suppose that \ouralgo solves \eqref{eq:w_} with mini-batch size $b$.
    Consider clients $i$ and $j$ in the same cluster $\cC$ of size $c$. Suppose  that $M_{ij}^2\in(0,\frac{1}{5})$, $b\ge \frac{2}{c^2}2L\eta(c-2)\sigma^2$ and $\zeta^2_{ik}\ge \norm{ \nabla f_i(\xx) + \nabla f_k(\xx) }_2^2$ for all $\xx$ and $k$.
    Let $\rho\ge \frac{\sqrt{3}L}{c}$ and step size  
    \begin{align*}
        \eta\le \min\left\{\frac{2}{\sigma\sqrt{LT}} \sqrt{ \frac{1}{c^2} \sum_{i,j\in\cC}  \left( \tilde{f}_{ij}\left(\zz_{ij}^{0}\right) - \tilde{f}_{ij}^\star  \right)}, \frac{1}{2\sqrt{3}L} \right\}.
    \end{align*}
    The consensus distance also converges to 0, i.e.
    \begin{align*}
        \frac{1}{c^2T}\sum_{t=0}^{T-1}\sum_{i,j\in\cC}  \E\left[\norm*{\xx_i^{t+1} - \xx_j^{t+1}}_2^2 \right]
        \le& \frac{6M^2_{ij}}{\rho^2 c^2} 
         \sqrt{ \frac{L\sigma^2}{c^2T} \sum_{i,j\in\cC} \left( \tilde{f}_{ij}\left(\zz_{ij}^{0}\right) - \tilde{f}_{ij}^\star  \right)}.
    \end{align*}
    Moreover, the gradient norm is upper bounded.
    \begin{align*}
       \frac{1}{c^2 T}\sum_{t=0}^{T-1} \sum_{i,j\in\cC} \E\left[\norm*{\nabla \tilde{f}_{ij}\left(\zz_{ij}^{t}\right)  }_2^2\right]
        \le& 3 \sqrt{ \frac{L\sigma^2}{c^2T} \sum_{i,j\in\cC} \left( \tilde{f}_{ij}\left(\zz_{ij}^{0}\right) - \tilde{f}_{ij}^\star  \right)}.
    \end{align*}
\end{theorem}
This theorem suggests that clients inside the same cluster gradually reach consensus. 
This cluster-level consensus model reaches stationary point of their losses by \eqref{eq:in_cluster}. Note that a larger penalization parameter $\rho$ and smaller values of $M_{ij}^2$ lead to faster convergence, which aligns with our expectations.
Note that $M_{ij}$ in \Cref{a:collaborative} measures how well i,j collaborate. A smaller $M_{ij}$ leads to better consensus distance in \Cref{theorem:main}, with $M_{ij}=0$ leading to identical data distribution. The following corollary states the convergence of norm of client gradient of model $\xx_i$.
\begin{corollary}\label{eq:corollary}
    Under same conditions as \Cref{theorem:main}, $\norm*{ \nabla f_i\left(\xx_{i}^{t}\right)}_2^2$ converges at a similar rate
    \begin{align*}
         \frac{1}{c^2T}\sum_{t=0}^{T-1} \sum_{i,j\in\cC} \norm*{ \nabla f_i\left(\xx_{i}^{t}\right)  }_2^2 
        \le & 4 \sqrt{ \frac{L\sigma^2}{c^2T} \sum_{i,j\in\cC} \left( \tilde{f}_{ij}\left(\zz_{ij}^{0}\right) - \tilde{f}_{ij}^\star  \right)}.
    \end{align*}
\end{corollary}

\section{Experiments}
\label{sec:exps}

 In this section, we present three experiments to demonstrate the practical effectiveness of \ouralgo. 
In the first two experiments, we benchmark \ouralgo in both a cross-silo federated learning setup involving 8 clients and a cross-device setup with 80 clients, using the CIFAR-100 dataset for multi-task learning~\citep{krizhevsky2009learning}. In the third experiment, we train language models in subsets of Wiki-40B data set, while learning domain weights within a simplex~\citep{guo-etal-2020-wiki}. Compared to state-of-the-art personalized federated learning baselines, \ouralgo obtains personalized models with higher quality and obtains correct cluster structures. Details of experiments, including the description of architectures and the system setup, are deferred to Appendix \ref{sec:exp-details}.

Throughout the experiments, we use popular federated learning baselines such as FedAvg~\citep{mcmahan2017communication}, Federated clustering (abbreviated as FC)~\citep{werner2023provably}, Ditto~\citep{smith2021ditto}, IFCA~\citep{ghosh2020efficient}, and the oracle algorithm. The oracle baseline definition varies in each setup, and will be discussed case by case. Note that we additionally pass to clustering-based algorithms, i.e. FC and IFCA, the actual number of clusters. Their experiment stats reported in this section, such as accuracy, perplexity, include such advantage.

\subsection{Cross-silo federated learning experiment with 8 clients}
\label{subsec:small-exp}

In this experiment, we evaluate the performance of \ouralgo by comparing the averaged accuracies of local models against those of established collaborative learning baselines. Our objective is to assess how effectively \ouralgo discerns and leverages the structure of data clusters relative to other collaborative learning algorithms. 

We simulate a cross-silo multi-task environment where a single model trained across all clients yields poor performance, thus highlighting the necessity for client selection. Our experimental configuration consists of 4 clusters, each containing 2 clients utilizing the ResNet-9 model \citep{he2015deep}. To encourage collaboration within clusters, we randomly allocate half of the dataset to each client in a cluster.
To differentiate between clusters, we introduce label diversity by flipping the image labels in each cluster using distinct random seeds. This process ensures that each class maintains unique labels throughout all clusters, effectively creating a scenario where a universally trained model would not be optimal, thereby necessitating personalized models that can cater to the specific label distribution of each cluster.

In this context, collaboration among clients within the same cluster is advantageous, as their datasets are complementary. There are two primary reasons why collaboration between different clusters may not be beneficial: (1) the dataset available to clients within each cluster is identical, negating the incentive to collaborate with clients from other clusters, and (2) the label flipping across clusters could mean that inter-cluster collaboration might actually degrade local model performance.

Given these considerations, we designate an oracle algorithm for our scenario: FedAvg, implemented separately within each cluster. This ensures that collaboration is confined to where it is most beneficial. Additionally, the oracle collaboration matrix is defined to be a block diagonal matrix, with entries of 1 for pairs of clients within the same cluster, indicating collaboration, and entries of 0 for pairs from different clusters, indicating no collaboration. This matrix serves as a benchmark for the ideal collaboration structure in our simulated environment.

To enable the practical application of \ouralgo, we sample pairs of clients in each iteration to update their collaboration weights. We begin by examining the impact of various sampling strategies on the performance of \ouralgo. The primary approach involves sampling with a constant probability of $\cO(1/n)$. Additionally, we observe that \ouralgo identifies an appropriate collaboration matrix early in the training process, motivating the use of a time-step-dependent sampling rate, $\cO(1/t)$. We also implement a mixed strategy: employing the constant sampling rate, $\cO(1/n)$, for the initial 0.2\% of iterations, followed by a switch to the time-dependent sampling rate, $\cO(1/t)$, for the remainder of the training. A comparison of these strategies with the non-sampling oracle, where all pairs are updated in every iteration, is presented in Table \ref{tab:sampling}. While \ouralgo demonstrates consistent performance across all sampling strategies, achieving results close to those of the non-sampling oracle, the mixed strategy shows a slight performance advantage.

\begin{table}[t]
\centering
\caption{Comparison of the average performance of \ouralgo across different sampling strategies for updating the weights of client pairs in the collaboration matrix. All strategies demonstrate performance close to that of the non-sampling oracle. However, the mixed strategy, which combines a constant sampling rate at the start with a time-dependent rate during later training phases, shows superior performance.
}
\vspace{0.5cm}
\label{tab:sampling}
{
\begin{tabular}{ccc}
\toprule
         & Acc.(\%) & Loss  \\ \midrule
Constant ($\cO(1/n)$) & 73.05 & 1.104 \\
Time-dependent ($\cO(1/t)$) & 73.18 & 1.226 \\
Mixed & 74.77 & 1.081 \\
No Sampling (Oracle) & 74.93 & 1.278 \\
\bottomrule
\end{tabular}
}
\end{table}

To further assess performance, we trained \ouralgo and other baseline algorithms for a total of 40,000 iterations.  The accuracy diagram is presented in Figure~\ref{subfig:acc8}. We can observe that \ouralgo almost preserves the bound in Oracle. Moreover, \ouralgo reaches the fixed accuracy of 60\% in 4500 iterations, that is 30\% faster than Ditto. For better comparison, the value of accuracy and loss are reported in Table \ref{tab:acc-and-loss}. Additionally, we can observe the evolution of the collaboration matrix for clustering algorithms and \ouralgo in Figure \ref{fig:heatmaps}. \ouralgo starts to find the clients with similar label permutation as early as 300 iterations, and stabilize in less than 5000 iterations (12.5\% of the training phase). IFCA always degenerate to one fully connected cluster. FC, on the other hand, periodically suffers from clustering mistakes even at the end of training.

In Figure \ref{subfig:cross-silo-ablation}, we present the results of the cross-silo experiment under various configurations to further assess the robustness of \ouralgo. First, we modify the fraction of the dataset allocated to each client. Intuitively, the total amount of data available to a cluster directly impacts the performance of \ouralgo. Then, we experiment with different numbers of clusters, each containing two clients, and observe that the number of clusters does not significantly affect \ouralgo's accuracy. Additionally, we investigate the effect of varying the number of clients per cluster while maintaining a fixed total of four clusters. In this setup, the dataset is partitioned among clients within each cluster, resulting in less data per client as cluster size increases. Despite this, \ouralgo leverages collaboration to maintain robust performance even with larger cluster sizes.

\begin{figure}[b!]
    \centering
        \begin{subfigure}[b]{0.33\textwidth}
    \centering
    \includegraphics[width=\textwidth]{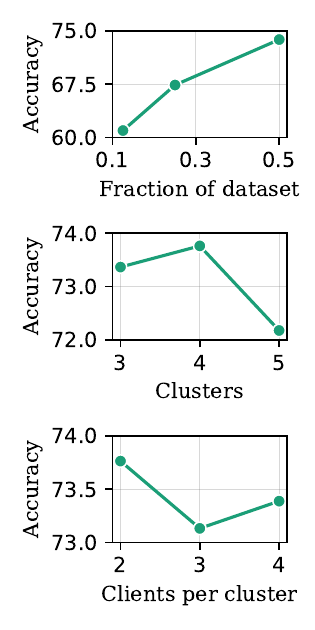}
    \caption{}
    \label{subfig:cross-silo-ablation}
    \end{subfigure}
    \hfill
    \begin{subfigure}[b]{0.65\textwidth}
    \centering
    \includegraphics[width=\textwidth]{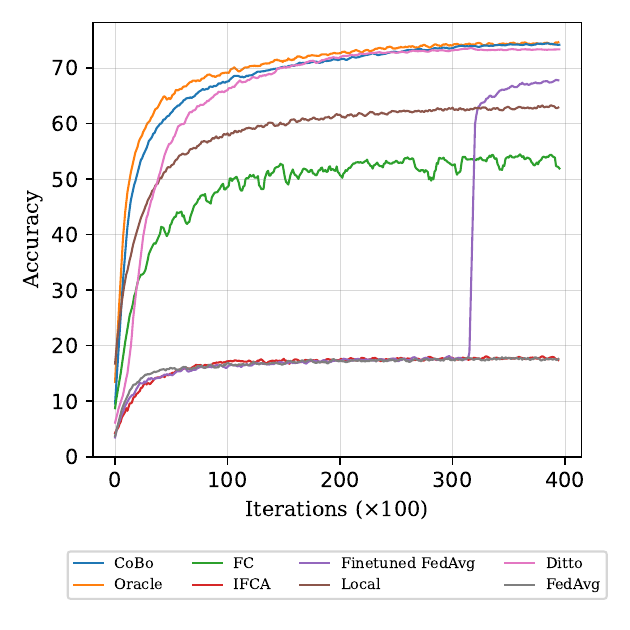}
    \caption{}
    \label{subfig:acc8}
    
    \end{subfigure}
    \caption{
    (\ref{subfig:cross-silo-ablation}) Average accuracy in cross-silo experiments with varying factors, including the fraction of the dataset available to clients, the number of clusters, and the number of clients per cluster. (\ref{subfig:acc8}) Average accuracy of personalized models for cross-silo federated learning with 8 clients. The "Oracle" denotes applying FedAvg to the clients with the same label permutation.  
    }
    
\end{figure}

\begin{figure}[t!]
    \centering
    \includegraphics[width=0.9\textwidth]{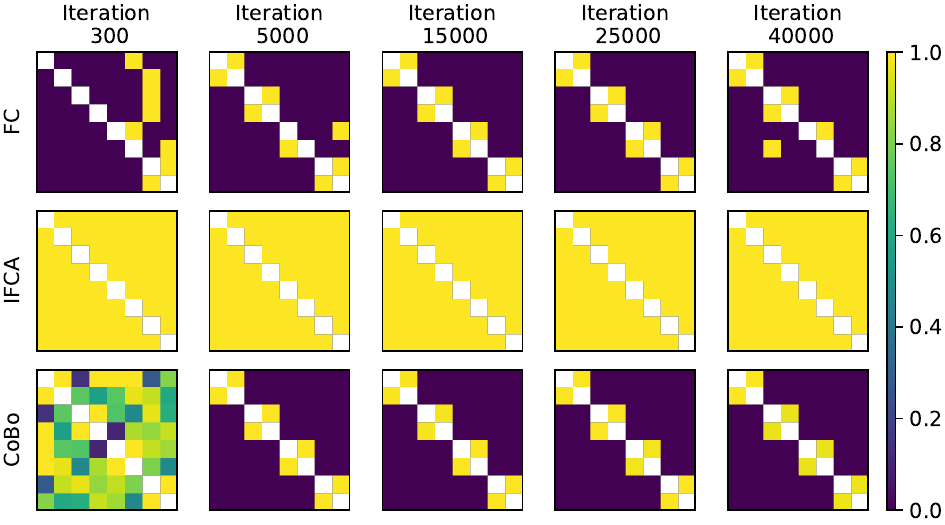}
    \caption{Collaboration matrices learned by Federated Clustering (FC), IFCA, and \ouralgo at different stages of training for cross-silo experiment with 8 clients. The diagonals are masked out. The oracle matrix is a block diagonal matrix with blocks of size 2. The collaboration matrix of \ouralgo already starts to look similar to oracle matrix within as low as 300 iterations (0.75\% of the total iterations), and converges to it within 5000 iterations (12.5\% of the total iterations). On the other hand, IFCA yields a fully-connected matrix while FC occasionally diverges from the achieved cluster structures (e.g., iterations 300, 5000, and 40000), even at the end of training.
    }
    \label{fig:heatmaps}
\end{figure}

\subsection{Cross-device experiment experiment with 80 clients}
\label{subsec:large-exp}

In this experiment, we demonstrate the performance of \ouralgo in a challenging cross-device federated learning setting with significant data heterogeneity. We create 10 clusters of varying sizes: 2 clusters consist of 6 clients each, another 2 comprise 7 clients each, and so on. Each cluster is allocated data from 10 distinct classes out of the total 100 classes available in the CIFAR-100 dataset, ensuring that the data across clusters are disjoint.
Within each cluster, the data are distributed uniformly at random among the clients. We then proceed to train individual ResNet-9 models \citep{he2015deep} owned by each client for a total of 20,000 iterations. This setup allows us to observe the behavior of \ouralgo and its ability to handle both the quantity and diversity of data across different client groups and cluster sizes.

We define the oracle algorithm and the corresponding collaboration matrix in the same manner as in Section \ref{subsec:small-exp}. Note that while we manually create the clusters, inter-cluster collaboration may still helpful in practice. It is impossible to know the actual groundtruth in this case.  
Consequently, we recognize that Oracle may not corresponds to the optimal performance. Nevertheless, this oracle still exhibits superior performance compared to other baselines that lack prior knowledge of the data distribution among clients, as evidenced by the results presented in Table \ref{tab:acc-and-loss}. The collaboration matrix and accuracy plots are differed to Figure \ref{fig:heatmaps80} and Figure \ref{fig:accuracy-80} in Appendix \ref{sec:exp-details}, respectively.

In this challenging experiment, \ouralgo surpasses all other baselines by at least 5.7\% in accuracy. This supports that \ouralgo scales well with the size of collaborative learning and exploits collaboration weights among clients at a fine-grain level.

 \subsection{Collaborative fine-tuning on language models}
 \label{subsec:lang-exp}
Recently, Large Language Models (LLMs) have become extremely popular due to their capability to effectively solve challenging tasks. Their downstream performances can be further enhanced by fine-tuning, however, the scarcity of data often yields to inferior performance, and necessitate collaboration~\citep{wagner2024personalizedcollaborativefinetuningondevice}. We therefore conduct an experiment of four clients, each having a pre-trained GPT-2 base model\footnote{https://github.com/karpathy/nanoGPT} with 124 million parameters in total~\citep{radford2019language}, and a subset of articles from the Wiki-40B dataset ~\citep{guo-etal-2020-wiki} with one of the four following languages: Catalan, Spanish, German, or Dutch. We use LoRA for Self-Attention and MLP layers for fine-tuning, which accounts for 0.47\% of the full parameters~\citep{hu2022lora}.

For data-hungry tasks, such as those involving LLMs, contributions from all domains are valuable. Clustering methods fall short in this aspect due to their binary, discrete outputs, which do not capture the nuanced degrees of collaboration needed. \ouralgo addresses this limitation by allowing for a continuous range of collaboration intensities, achieved by a simple yet effective modification to the projection domain in  \eqref{eq:w}. Specifically, we employ a probability simplex, denoted as
 $\Delta_i = \{w_{ij} \geq 0, \sum_{j}w_{ij} = 1\}$ as the domain of inner problem.

In Table \ref{tab:acc-and-loss} we compare the perplexity of \ouralgo with baselines after 500 iterations, when FedAvg converges. There is no oracle domain weights in this experiment due to the complicated coherence of languages. We therefore drop oracle algorithm in the table. 
\ouralgo obtains the best perplexity among all algorithms. In Figure \ref{fig:domain-weights}, we demonstrate the domain weights learned for the Catalan language. Overall, Catalan gives the highest collaboration weight to Spanish, which is reasonable considering the similarity between two languages. 

\begin{figure}[t!]
    \centering
    \includegraphics[width=0.9\textwidth]{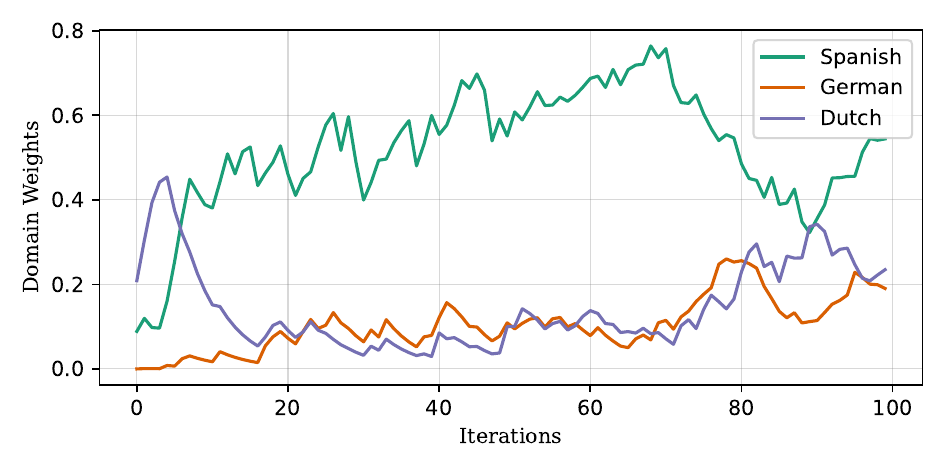}
    \caption{Domain weights found by \ouralgo for Catalan language. There are 4 domains in total: Catalan, Spanish, German, and Dutch. The curves are smoothed by exponential moving average.
    }
    \label{fig:domain-weights}
\end{figure}

\begin{table}[t]
\centering
\caption{Comparisons of model quality and fairness measure of personalized models for cross-silo experiment with 8 clients, and cross-device experiment with 80 clients, and the language modelling experiment with 4 clients having different languages. Federated clustering (FC) is not scalable with number of clients due to its $\cO(n^2)$ complexity, and therefore ignored in the cross-device fl experiment. The clustering algorithms IFCA and FC are not applicable to LLMs and there ignored. Note that Oracle is not defined in the LLMs experiment. The column ``Imp.(\%)'' demonstrates the percentage of clients with improved performance compared to local training.
}
\vspace{0.5cm}
\label{tab:acc-and-loss}
{
\resizebox{\textwidth}{!}{
\begin{tabular}{cccccccccc}
\toprule
       &   \multicolumn{3}{c}{Cross-silo} &  \multicolumn{3}{c}{Cross-device} & \multicolumn{2}{c}{Fine-tuning LLMs} \\\cmidrule(r){2-4} \cmidrule(r){5-7} \cmidrule(r){8-9}
       &  Acc.(\%) & Loss  & Imp.(\%) & Acc.(\%) & Loss & Imp.(\%) & Perplexity & Imp.(\%) \\ \midrule
Local  &  64.9 ± \footnotesize{0.1} & 1.67 & -  & 54.9 ± \footnotesize{0.1} &  1.40 & - & 41.26 ± \footnotesize{0.38} & -  \\
FedAvg & 18.8 ± \footnotesize{0.1} & 2.66 & 0 & 53.9 ± \footnotesize{0.1} & 1.79 & 29 & 64.84 ± \footnotesize{0.00} & 0\\
Fine-tuning FedAvg & 70.2 ± \footnotesize{0.2} & 1.77 & 0 & 58.9 ± 0.1 & 1.88 & 94 & 46.70 ± \footnotesize{0.07} & 0 \\
Ditto & 73.5 ± \footnotesize{0.3} & 1.55 & \textbf{100} & 70.3 ± \footnotesize{0.1} & 1.21 & \textbf{100} & 40.05 ± \footnotesize{0.01} & \textbf{100}  \\
IFCA & 18.6 ± \footnotesize{0.1}  &  2.75 & 0 & 45.6 ± \footnotesize{0.8}  & 2.15 & 4 & - & - \\
FC & 55.1 ± \footnotesize{0.4} & 1.79 & 0 & - & - & - & - & -  \\
\ouralgo & \textbf{74.6 ± \footnotesize{0.2}} & \textbf{1.08} & \textbf{100} & \textbf{79.6 ± \footnotesize{0.4}} & \textbf{0.97} & \textbf{100} & \textbf{39.28 ± \footnotesize{0.01}} & \textbf{100} \\
\midrule
Oracle & 75.4 ± \footnotesize{0.2} &  1.07 & 100 & 83.6 ± \footnotesize{0.3} & 0.70 & 100 & - & - \\
\bottomrule
\end{tabular}
}
}
\end{table}

\section{Related Work}

\paragraph{Personalized federated learning.}   Personalized federated learning has received significant attention in recent years due to its potential to tailor models to individual user data while benefit from collaboration~\cite{sattler2019,t2020personalized, tan2022towards,fortfederated, kulkarni2020survey}. There are various flavors of personalized federated learning. \textsc{Ditto} trains personalized models by incorporating a regularization term that penalizes the divergence from a global model~\citep{smith2021ditto}. 
Many personalization works assume that clients are drawn from clusters. For example, \citet{marfoq2021mixture} use K-nearest neighbors (KNN) to determine collaborators. \citet{mansour2020, ghosh2020efficient, werner2023provably} develop $K$ personalized models and assign clients to clusters based on criteria such as minimum function values or gradient similarities. Additionally, \citet{even2022sample} provided theoretical insights by establishing lower bounds, which demonstrate that the optimal gradient filtering strategy involves clustering clients with identical optima.

\paragraph{Federated Learning with Client Selection}
In federated learning, client selection is often performed by simultaneously minimizing task losses and collaborative weights in a single-level objective function. \citet{zantedeschi2020fully} minimize task losses augmented with a penalization term $w_{ij}\norm{\xx_i-\xx_j}_2^2$, similar to our outer problem. However, optimizing $w_{ij}$ directly can lead to a degenerate solution ($w_{ij}=0$), which necessitates an additional penalization for small $w_{ij}$ values.
\citet{smith2017federated} approach multi-task learning by minimizing task losses with a more sophisticated penalization term that accounts for the relationships between tasks. This formulation requires the client-selection function to be consistent with client selection, which can negatively impact performance.
Apart from multi-task federated learning, a similar bilevel optimization formulation has been used by \citet{le2023refined} to find a sparse mixing matrix while training a consensus model in the outer problem.

\paragraph{Bilevel optimization and alternating optimization.} Bilevel optimization is a powerful tool which models a broad range of problems, such as reinforcement learning~\cite{dai2018sbeed,nachum2020reinforcement,huang2021biadam,hu2020biased}, and linearly-solvable Markov decision process~\cite{dai2017learning}, meta-learning~\cite{finn2017model}, etc. A typical bilevel optimization problem, as the name indicates, consists of an outer and an inner optimization problem whose variables are inter-dependent. Typical bilevel optimization solvers requires hessian information which is usually expansive to acquire ~\citep{finn2017model}. On the other hand, alternating optimization tools has been used be used to solve bilevel optimization problem~\citep{bezdek2003convergence,chen2021closing}. While in general there is no universal convergence guarantees for alternative optimizations, the special structure of our inner problem ensures the convergence of \ouralgo to the stationary point.

\section{Conclusions}

Existing collaborative learning algorithms only allow coarse-grained collaboration, which leads to inferior performance in practice. To address this issue, we model collaborative learning as a special bilevel optimization problem  
where client selection is based on the optimization of a linear function of gradient alignment measure for each pair of clients. 
In addition, we propose an efficient SGD-type alternating optimization algorithm \ouralgo which is scalable, elastic, and enjoy theoretical guarantees. Besides, \ouralgo empirically outperforms popular personalized federated learning algorithms in realistic collaborative learning problems.

\paragraph{Limitations.} Due to time constraints, the proof of \Cref{theorem:main} is provided for a simplified scenario in which the inner problem employs the full gradient as per \eqref{eq:w}, and the minimization problem is solved exactly. Additionally, in the cross-device federated learning experiments involving 80 clients, each algorithm was executed only once due to the same time limitations. We will conduct repeated experiments to ensure the robustness of our findings in future iterations of this work.

\bibliography{report}

\clearpage

\appendix

\section{Theory}

Let $\zz_{ij}^t := \frac{1}{2}(\xx_i^{t}+\xx_j^{t})$ be the average iterate of $\xx_i^{t}$ and $\xx_j^{t}$ and $\tilde{f}_{ij}:= \frac{1}{2} (f_i + f_j)$ be their averaged objective.  
\begin{lemma}\label{lemma:sd}
    Suppose \Cref{a:smoothness} hold true. Let $\eta\le\frac{1}{2L}$. Then for $i,j$ in the same cluster $\cC$ of size $c$
    \begin{align*}
        \sum_{i,j\in\cC}\norm*{\nabla \tilde{f}_{ij}\left(\zz_{ij}^{t}\right)  }_2^2
        \le& \frac{2}{\eta} \sum_{i,j\in\cC}\left( \tilde{f}_{ij}\left(\zz_{ij}^{t}\right) - \tilde{f}_{ij}\left(\zz_{ij}^{t+1}\right)  \right) \\
        & +  \left(\frac{3L^2}{4} + 3c^2\rho^2 + \frac{ L\eta\rho^2 (c-2) 2 \sigma^2}{b} \right) \sum_{i,j\in\cC} \norm*{\xx_i^t-\xx_j^t}_2^2 \\
        & + 3nc\rho^2 \sum_{i\in\cC} \sum_{k=1}^n |w_{ik}^{t+1} - w_{ik}^\star|^2 \norm*{\xx_i^t - \xx_k^t}_2^2 + \frac{c^2 L\eta \sigma^2}{2}.
    \end{align*}
\end{lemma}

\subsection{Proof of \Cref{lemma:sd}}
\begin{proof} 
Let ${h}_i^t$ and ${h}_j^t$ be independent and unbiased estimates of $\nabla f_i(z_{ij}^t)$, $\nabla f_j(z_{ij}^t)$ respectively. The variance of ${h}_i^t$ has a variance of $\frac{\sigma^2}{b}$. Let's denote $\mathbb{E}_g[\cdot]:=\mathbb{E}_{g_1,\ldots, g_n}[\cdot|\zz_i^t]$ and $\mathbb{E}_h[\cdot]:=\mathbb{E}_{h_1,\ldots, h_n}[\cdot]$ and let $\mathbb{E}[\cdot]=\mathbb{E}_h[\mathbb{E}_g[\cdot]]$. 
    By the L-smoothness assumption~\Cref{a:smoothness} and bounded noise assumption~\Cref{a:noise-bound}
\begin{align*}
    \mathbb{E}_h\mathbb{E}_g\left[\tilde{f}_{ij}\left(z_{ij}^{t+1}\right)\right]
    \le& \tilde{f}_{ij}\left(z_{ij}^{t}\right) + 
    \left\langle \nabla \tilde{f}_{ij}\left(z_{ij}^{t}\right), \mathbb{E}_h\mathbb{E}_g\left[z_{ij}^{t+1} - z_{ij}^{t}\right] \right\rangle + \frac{L}{2} \lVert\mathbb{E}_h\mathbb{E}_g[z_{ij}^{t+1} - z_{ij}^{t}]\rVert_2^2 \\
    & + \frac{L}{2}\underbrace{ \mathbb{E}_h\mathbb{E}_g\left[\lVert z_{ij}^{t+1} - z_{ij}^{t} - \mathbb{E}_h\mathbb{E}_g[z_{ij}^{t+1} - z_{ij}^{t}] \rVert_2^2 \right]}_{\cT_{ij}}.
\end{align*}
Here
    \begin{align*}
        \E\left[\tilde{f}_{ij}\left(\zz_{ij}^{t+1}\right)\right]
        \le& \E\left[\tilde{f}_{ij}\left(\zz_{ij}^{t}\right)
        + \left\langle \nabla \tilde{f}_{ij}\left(\zz_{ij}^{t}\right), \zz_{ij}^{t+1} - \zz_{ij}^{t} \right\rangle + \frac{L}{2} \norm*{\zz_{ij}^{t+1} - \zz_{ij}^{t}}_2^2\right] + \frac{L}{2}\cT_{ij} \\
        =&\tilde{f}_{ij}\left(\zz_{ij}^{t}\right) + 
        \left\langle \nabla \tilde{f}_{ij}\left(\zz_{ij}^{t}\right), \E\left[\zz_{ij}^{t+1} - \zz_{ij}^{t}\right] \right\rangle + \frac{L}{2} \norm*{\E[\zz_{ij}^{t+1} - \zz_{ij}^{t}]}_2^2 \\
        & + \frac{L}{2} \E\left[\norm*{\zz_{ij}^{t+1} - \zz_{ij}^{t} - \E[\zz_{ij}^{t+1} - \zz_{ij}^{t}]}_2^2 \right] + \frac{L}{2}\cT_{ij}\\
        =&\tilde{f}_{ij}\left(\zz_{ij}^{t}\right) + 
        \left\langle \nabla \tilde{f}_{ij}\left(\zz_{ij}^{t}\right), \E\left[\zz_{ij}^{t+1} - \zz_{ij}^{t}\right] \right\rangle + \frac{L}{2} \norm*{\E[\zz_{ij}^{t+1} - \zz_{ij}^{t}]}_2^2 \\
        & + \frac{L\eta^2}{8} \E\left[\norm*{\gg_i^{t+1} + \gg_j^{t+1} - (\nabla f_i(\xx_i^t) + \nabla f_j(\xx_j^t) ) }_2^2 \right] + \frac{L}{2}\cT_{ij}\\
        \le&\tilde{f}_{ij}\left(\zz_{ij}^{t}\right) + 
        \left\langle \nabla \tilde{f}_{ij}\left(\zz_{ij}^{t}\right), \E\left[\zz_{ij}^{t+1} - \zz_{ij}^{t}\right] \right\rangle \\ 
        & + \frac{L}{2} \norm*{\E[\zz_{ij}^{t+1} - \zz_{ij}^{t}]}_2^2 
        + \frac{L\eta^2\sigma^2}{4} + \frac{L}{2}\cT_{ij}.
    \end{align*}
    Expand the inner product with equality $-\langle \xx, \yy\rangle = -\frac{1}{2} \norm{\xx}_2^2 - \frac{1}{2} \norm{\yy}_2^2 + \frac{1}{2} \norm{\xx-\yy}_2^2 $
    \begin{align*}
        \E\left[\tilde{f}_{ij}\left(\zz_{ij}^{t+1}\right)\right]
        \le& \tilde{f}_{ij}\left(z_{ij}^{t}\right) 
         + \frac{\eta}{2} \norm*{\frac{\E[\zz_{ij}^{t+1}] - \zz_{ij}^{t}}{\eta} + \nabla \tilde{f}_{ij}\left(\zz_{ij}^{t}\right) }_2^2 
         - \frac{\eta}{2} \norm*{\frac{\E[\zz_{ij}^{t+1}] - \zz_{ij}^{t}}{\eta}  }_2^2  \\
         &- \frac{\eta}{2} \norm*{\nabla \tilde{f}_{ij}\left(\zz_{ij}^{t}\right)  }_2^2  
         + \frac{L\eta^2}{2} \norm*{\frac{\E[\zz_{ij}^{t+1}] - \zz_{ij}^{t}}{\eta}  }_2^2 +  \frac{L\eta^2\sigma^2}{4}  + \frac{L}{2} \cT_{ij}\\
        \le& \tilde{f}_{ij}\left(\zz_{ij}^{t}\right) 
         + \frac{\eta}{2} \norm*{\frac{\E[\zz_{ij}^{t+1}] - \zz_{ij}^{t}}{\eta} + \nabla \tilde{f}_{ij}\left(\zz_{ij}^{t}\right) }_2^2 \\
         &- \frac{\eta}{4} \norm*{\frac{\E[\zz_{ij}^{t+1}]- \zz_{ij}^{t}}{\eta}  }_2^2  
         - \frac{\eta}{2} \norm*{\nabla \tilde{f}_{ij}\left(\zz_{ij}^{t}\right)  }_2^2 + \frac{L\eta^2\sigma^2}{4} + \frac{L}{2} \cT_{ij}
    \end{align*}
    where $\eta\le\frac{1}{2L}$ is applied in the last inequality. Then
    \begin{align*}
        \norm*{\nabla \tilde{f}_{ij}\left(\zz_{ij}^{t}\right)  }_2^2
        \le& \frac{2}{\eta} \left( \tilde{f}_{ij}\left(\zz_{ij}^{t}\right) - \E[\tilde{f}_{ij}\left(\zz_{ij}^{t+1}\right) ] \right)
        + \frac{L\eta\sigma^2}{2}
        + \underbrace{\norm*{\frac{\E[\zz_{ij}^{t+1}] - \zz_{ij}^{t}}{\eta} + \nabla \tilde{f}_{ij}\left(\zz_{ij}^{t}\right) }_2^2}_{=:\cT} \\
        &+ \frac{2}{\eta}\frac{L}{2}\cT_{ij}.
    \end{align*}
    The $\cT$ expand the iterate $\xx_i^{t+1}$ using \eqref{eq:x}
    \begin{align*}
    \mathcal{T}:=&\norm*{\frac{\E[\zz_{ij}^{t+1}] - \zz_{ij}^{t}}{\eta} + \nabla \tilde{f}_{ij}\left(\zz_{ij}^{t}\right) }_2^2 \\
    =& \norm*{
    \nabla \tilde{f}_{ij}\left(\zz_{ij}^{t}\right)
    - \frac{1}{2} \left(
        \nabla f_i(\xx_i^t) + \nabla f_j(\xx_j^t)
    \right)
    -  \frac{\rho}{2} \left(\sum_{k=1}^n w_{ik}^{t+1} (\xx_i^t - \xx_k^t)
     +\sum_{k=1}^n w_{jk}^{t+1} (\xx_j^t - \xx_k^t) \right)
    }_2^2  \\
    \le& 3\underbrace{\norm*{
    \nabla \tilde{f}_{ij}\left(
\zz_{ij}^{t}\right)
    -  \frac{1}{2} \left(\nabla f_i(\xx_i^t) + \nabla f_j(\xx_j^t)\right)
    }_2^2}_{\mathcal{T}_1}
    + 3\underbrace{\norm*{
     \frac{\rho}{2} \left(\sum_{k=1}^n w_{ik}^{\star} (\xx_i^t - \xx_k^t)
     +\sum_{k=1}^n w_{jk}^{\star} (\xx_j^t - \xx_k^t) \right)
    }_2^2}_{\mathcal{T}_2} \\
    &+ 3\underbrace{\norm*{
     \frac{\rho}{2} \left(\sum_{k=1}^n (w_{ik}^{t+1}-w_{ik}^\star) (\xx_i^t - \xx_k^t)
     +\sum_{k=1}^n (w_{jk}^{t+1}-w_{jk}^\star) (\xx_j^t - \xx_k^t) \right)
    }_2^2}_{\mathcal{T}_3}.
    \end{align*}
    \noindent \textbf{Bound $\mathcal{T}_1$: } Use L-smoothness of $f_i$ and $f_j$. Take expectation with respect to  $\gg_i^t$ and $\gg_j^t$ which are unbiased estimates of $\nabla f_i(\xx_i^t)$ and $\nabla f_j(\xx_j^t)$
    \begin{align*}
        \E[\mathcal{T}_1]
        =& \norm*{\nabla \tilde{f}_{ij}\left(\zz_{ij}^{t}\right)
    -  \frac{1}{2} \left(\nabla f_i(\xx_i^t) + \nabla f_j(\xx_j^t) \right)}_2^2 \\
        \le& \frac{L^2}{2} \norm*{\zz^t_{ij} - \xx_i^t}_2^2 + \frac{L^2}{2} \norm*{\zz^t_{ij} - \xx_j^t}_2^2 \\
        =& \frac{L^2}{4}\norm*{\xx_i^t-\xx_j^t}_2^2.
    \end{align*}

    \noindent \textbf{Bound $\mathcal{T}_2$: }  Use Cauchy-Schwarz inequality and $\cC$ has a cluster size of $c$
    \begin{align*}
        \mathcal{T}_2 \le \frac{c\rho^2}{2} \left( \sum_{k=1}^n w_{ik}^\star \norm*{\xx_i^t - \xx_k^t}_2^2
        +  \sum_{k=1}^n w_{jk}^\star \norm*{\xx_j^t - \xx_k^t}_2^2 \right).
    \end{align*}

    \noindent \textbf{Bound $\mathcal{T}_3$: }  Use Cauchy-Schwarz inequality
    \begin{align*}
        \mathcal{T}_3 \le \frac{n\rho^2}{2}  \left(\sum_{k=1}^n |w_{ik}^{t+1} - w_{ik}^\star|^2 \norm*{\xx_i^t - \xx_k^t}_2^2
        +  \sum_{k=1}^n |w_{jk}^{t+1} - w_{jk}^\star|^2 \norm*{\xx_j^t - \xx_k^t}_2^2 \right).
    \end{align*}

    Sum over all of the $i,j$  in the same cluster $\cC$ yields
    \begin{align*}
        \sum_{i,j\in\cC}\norm*{\nabla \tilde{f}_{ij}\left(\zz_{ij}^{t}\right)  }_2^2
        \le& \frac{2}{\eta} \sum_{i,j\in\cC}\left( \tilde{f}_{ij}\left(\zz_{ij}^{t}\right) - \E[\tilde{f}_{ij}\left(\zz_{ij}^{t+1}\right) ] \right)
        + \frac{c^2L\eta\sigma^2}{2}
        +  \frac{3L^2}{4}\sum_{i,j\in\cC} \norm*{\xx_i^t-\xx_j^t}_2^2 
        \\
        & + 3c^2\rho^2  \sum_{i,j\in\cC} \norm*{\xx_i^t - \xx_j^t}_2^2 
        + 3nc\rho^2 \sum_{i\in\cC} \sum_{k=1}^n |w_{ik}^{t+1} - w_{ik}^\star|^2 \norm*{\xx_i^t - \xx_k^t}_2^2 \\
        & + \frac{2}{\eta} \frac{L}{2} \sum_{i,j\cC} \cT_{ij} \\
        \le& \frac{2}{\eta} \sum_{i,j\in\cC}\left( \tilde{f}_{ij}\left(\zz_{ij}^{t}\right) - \tilde{f}_{ij}\left(\zz_{ij}^{t+1}\right)  \right)
        +  \left(\frac{3L^2}{4} + 3c^2\rho^2\right) \sum_{i,j\in\cC} \norm*{\xx_i^t-\xx_j^t}_2^2 \\
        & + 3nc\rho^2 \sum_{i\in\cC} \sum_{k=1}^n |w_{ik}^{t+1} - w_{ik}^\star|^2 \norm*{\xx_i^t - \xx_k^t}_2^2 + \frac{c^2L\eta\sigma^2}{2} 
        + \frac{2}{\eta}\frac{L}{2} \sum_{i,j\cC} \cT_{ij}.
    \end{align*}

    Note that $\cT_{ij}$ can be bounded as follows
    \begin{align*}
        & \mathbb{E}_h\mathbb{E}_g \left[\lVert z_{ij}^{t+1} - z_{ij}^{t} - \mathbb{E}_h\mathbb{E}_g[z_{ij}^{t+1} - z_{ij}^{t}] \rVert_2^2 \right] \\
        =& \mathbb{E}_h\mathbb{E}_g \left[\lVert z_{ij}^{t+1} - z_{ij}^{t} \pm \mathbb{E}_h[z_{ij}^{t+1} - z_{ij}^{t}] - \mathbb{E}_h\mathbb{E}_g[z_{ij}^{t+1} - z_{ij}^{t}] \rVert_2^2 \right] \\ 
        =&  \mathbb{E}_h\mathbb{E}_g \left[\lVert \mathbb{E}_h[z_{ij}^{t+1} - z_{ij}^{t}] - \mathbb{E}_h\mathbb{E}_g[z_{ij}^{t+1} - z_{ij}^{t}] \rVert_2^2 \right]+ \mathbb{E}_h\mathbb{E}_g \left[\lVert z_{ij}^{t+1} - z_{ij}^{t} - \mathbb{E}_h[z_{ij}^{t+1} - z_{ij}^{t}]  \rVert_2^2 \right].
    \end{align*}
    Plug in the above equality to the above inequality
    \begin{align*}
        \mathbb{E}_h\mathbb{E}_g\left[\tilde{f}_{ij}\left(z_{ij}^{t+1}\right)\right]
        \le& \tilde{f}_{ij}\left(z_{ij}^{t}\right) + 
        \left\langle \nabla \tilde{f}_{ij}\left(z_{ij}^{t}\right), \mathbb{E}_h\mathbb{E}_g\left[z_{ij}^{t+1} - z_{ij}^{t}\right] \right\rangle + \frac{L}{2} \lVert\mathbb{E}_h\mathbb{E}_g[z_{ij}^{t+1} - z_{ij}^{t}]\rVert_2^2 \\
        & + \mathbb{E}_h\mathbb{E}_g \left[\lVert \mathbb{E}_h[z_{ij}^{t+1} - z_{ij}^{t}] - \mathbb{E}_h\mathbb{E}_g[z_{ij}^{t+1} - z_{ij}^{t}] \rVert_2^2 \right] \\
        & + \mathbb{E}_h\mathbb{E}_g \left[\lVert z_{ij}^{t+1} - z_{ij}^{t} - \mathbb{E}_h[z_{ij}^{t+1} - z_{ij}^{t}]  \rVert_2^2 \right].
    \end{align*}
    The last term can be expanded as follows
    \begin{align*}
        &\mathbb{E}_h\mathbb{E}_g \left[\lVert z_{ij}^{t+1} - z_{ij}^{t} - \mathbb{E}_h[z_{ij}^{t+1} - z_{ij}^{t}]  \rVert_2^2 \right] \\
        =&\mathbb{E}_h\mathbb{E}_g \left[ \left\lVert \frac{\eta}{2} (g_i^t + g_j^t) + \frac{\eta\rho}{2}\sum_{k=1}^n (w_{ik}^{t+1} (x_i^t - x_k^t) + w_{jk}^{t+1} (x_j^t - x_k^t)) \right.\right. \\
        & \qquad\left.\left. - \mathbb{E}_h \left[\frac{\eta}{2} (g_i^t + g_j^t) + \frac{\eta\rho}{2}\sum_{k=1}^n (w_{ik}^{t+1} (x_i^t - x_k^t) + w_{jk}^{t+1} (x_j^t - x_k^t)) \right]  \right\rVert_2^2 \right] \\
        =&\mathbb{E}_h\left[ \left\lVert \frac{\eta\rho}{2}\sum_{k=1}^n ((w_{ik}^{t+1} - \mathbb{E}_h[w_{ik}^{t+1}]) (x_i^t - x_k^t) + (w_{jk}^{t+1}- \mathbb{E}_h[w_{jk}^{t+1}]) (x_j^t - x_k^t))  \right\rVert_2^2 \right] \\
        =&\frac{\eta^2\rho^2}{4} \sum_{k\neq i,j} \left(\mathbb{E}_h \left[ \left\lVert w_{ik}^{t+1} - \mathbb{E}_h[w_{ik}^{t+1}] \right\rVert^2_2 \right]  \lVert x_i^t - x_k^t\rVert_2^2 + \mathbb{E}_h \left[ \left\lVert w_{jk}^{t+1} - \mathbb{E}_h[w_{jk}^{t+1}] \right\rVert^2_2 \right]  \lVert x_j^t - x_k^t\rVert_2^2 \right)
    \end{align*}
    where we use the independence of random variables in the last equality. Average over i, j yields
    \begin{align*}
    &\frac{1}{c^2}\sum_{ij}\mathbb{E}_h\mathbb{E}_g \left[\lVert z_{ij}^{t+1} - z_{ij}^{t} - \mathbb{E}_h[z_{ij}^{t+1} - z_{ij}^{t}]  \rVert_2^2 \right] \\
    =&\frac{\eta^2\rho^2(c-2)}{2c^2} \sum_{i,j} \mathbb{E}_h \left[ \left\lVert w_{ij}^{t+1} - \mathbb{E}_h[w_{ij}^{t+1}] \right\rVert^2_2 \right]  \lVert x_i^t - x_j^t\rVert_2^2  \\
    \le&\frac{\eta^2\rho^2(c-2)}{2c^2} \frac{4\sigma^2}{b} \sum_{i,j}  \lVert x_i^t - x_j^t\rVert_2^2.
    \end{align*}
    Then
    \begin{align*}
        \sum_{i,j\in\cC}\norm*{\nabla \tilde{f}_{ij}\left(\zz_{ij}^{t}\right)  }_2^2
        \le& \frac{2}{\eta} \sum_{i,j\in\cC}\left( \tilde{f}_{ij}\left(\zz_{ij}^{t}\right) - \tilde{f}_{ij}\left(\zz_{ij}^{t+1}\right)  \right)
        +  \left(\frac{3L^2}{4} + 3c^2\rho^2\right) \sum_{i,j\in\cC} \norm*{\xx_i^t-\xx_j^t}_2^2 \\
        & + 3nc\rho^2 \sum_{i\in\cC} \sum_{k=1}^n |w_{ik}^{t+1} - w_{ik}^\star|^2 \norm*{\xx_i^t - \xx_k^t}_2^2 + \frac{c^2L\eta\sigma^2}{2} \\
        & + \frac{2}{\eta}\frac{L}{2} \frac{\eta^2\rho^2(c-2)}{2} \frac{4\sigma^2}{b} \sum_{i,j}  \lVert x_i^t - x_j^t\rVert_2^2 \\
        =& \frac{2}{\eta} \sum_{i,j\in\cC}\left( \tilde{f}_{ij}\left(\zz_{ij}^{t}\right) - \tilde{f}_{ij}\left(\zz_{ij}^{t+1}\right)  \right) \\
        & +  \left(\frac{3L^2}{4} + 3c^2\rho^2 + \frac{ L\eta\rho^2 (c-2) 2 \sigma^2}{b} \right) \sum_{i,j\in\cC} \norm*{\xx_i^t-\xx_j^t}_2^2 \\
        & + 3nc\rho^2 \sum_{i\in\cC} \sum_{k=1}^n |w_{ik}^{t+1} - w_{ik}^\star|^2 \norm*{\xx_i^t - \xx_k^t}_2^2 + \frac{c^2L\eta\sigma^2}{2}.
    \end{align*}
\end{proof}
\begin{lemma}\label{lemma:consensus_distance}
    Suppose $M_{ij}\le\frac{1}{5}$. Let $\rho\ge \frac{\sqrt{3}L}{c}$ and $\eta\le\frac{1}{2\rho c}\le\frac{1}{2\sqrt{3}L}$ then
    \begin{align*}
        \sum_{i,j\in\cC}  \norm*{\xx_i^{t+1} - \xx_j^{t+1}}_2^2 
        \le& \left(1-\eta\rho c\right) \sum_{i,j\in\cC}\norm*{\xx_i^{t} - \xx_j^{t}}_2^2 \\
        &+5n\eta\rho\sum_{i\in\cC} \sum_{k=1}^n|w_{ik}^t-w_{ik}^\star|^2 \norm{\xx_i^t - \xx_k^t}_2^2 \\
        &+ \frac{6 M^2_{ij}}{\rho c} 
        \sum_{i,j\in\cC}\left( \tilde{f}_{ij}\left(\zz_{ij}^{t}\right) - \tilde{f}_{ij}\left(\zz_{ij}^{t+1}\right)  \right)
        +  \frac{3\eta M^2_{ij}}{\rho c}  \frac{c^2 L\eta \sigma^2}{2}.
    \end{align*}

\end{lemma}
\begin{proof}
    Expand $\xx_i^{t+1} - \xx_j^{t+1}$ with \eqref{eq:x}
    \begin{align*}
        \xx_i^{t+1} - \xx_j^{t+1}
        =\xx_i^{t} - \xx_j^{t} - \eta\rho \sum_{k=1}^n \left(w_{ik}^{t+1} (\xx_i^t - \xx_k^t) - w_{jk}^{t+1} (\xx_j^t - \xx_k^t) \right)
        - \eta\left(\nabla f_i(\xx_i^t) - \nabla f_j(\xx_j^t)\right).
    \end{align*}
    As $i$ and $j$ belong to the same cluster (i.e., $w_{ij}^\star=1$), we add $\pm2\eta\rho \sum_{k=1}^n w_{ik}^\star (\xx_i^t - \xx_k^t)$
    \begin{align*}
        \xx_i^{t+1} - \xx_j^{t+1}=&(1-2\eta\rho c) (\xx_i^{t} - \xx_j^{t})
        - \eta\rho \sum_{k=1}^n \left((w_{ik}^{t+1}-w_{ik}^\star) (\xx_i^t - \xx_k^t) - (w_{jk}^{t+1}-w_{jk}^\star) (\xx_j^t - \xx_k^t) \right) \\
        &- \eta\left(\nabla f_i(\xx_i^t) - \nabla f_j(\xx_j^t)\right).
    \end{align*}
    Compute the norm of $\xx_i^{t+1} - \xx_j^{t+1}$ and choose $\eta\rho\le \frac{1}{2c}$ to use Jensen's inequality
    \begin{align*}
        \norm*{\xx_i^{t+1} - \xx_j^{t+1}}_2^2 
        \le& 
         (1-2\eta\rho c)\norm*{\xx_i^{t} - \xx_j^{t}}_2^2 \\
         &+ 2\eta\rho c\left\lVert \frac{1}{2c} \sum_{k=1}^n \left((w_{ik}^{t+1}-w_{ik}^\star) (\xx_i^t - \xx_k^t) - (w_{jk}^{t+1}-w_{jk}^\star) (\xx_j^t - \xx_k^t) \right) \right. \\
        &\qquad \qquad + \left.\frac{1}{2\rho c}\left(\nabla f_i(\xx_i^t) - \nabla f_j(\xx_j^t)\right)\right\rVert_2^2.
    \end{align*}
    Expand the right-hand side with Cauchy-Schwarz inequality
    \begin{align*}
        \norm*{\xx_i^{t+1} - \xx_j^{t+1}}_2^2 
        \le& 
         (1-2\eta\rho c)\norm*{\xx_i^{t} - \xx_j^{t}}_2^2 \\
         &+ 4\eta\rho c\norm*{ \frac{1}{2c} \sum_{k=1}^n \left((w_{ik}^t-w_{ik}^\star) (\xx_i^t - \xx_k^t) - (w_{jk}^t-w_{jk}^\star) (\xx_j^t - \xx_k^t) \right)}_2^2 \\
        &+ 4\eta\rho c\norm*{  \frac{1}{2\rho c}\left(\nabla f_i(\xx_i^t) - \nabla f_j(\xx_j^t)\right)}_2^2 \\
        \le& 
         (1-2\eta\rho c)\norm*{\xx_i^{t} - \xx_j^{t}}_2^2
         + 8\eta\rho c\norm*{ \frac{1}{2c} \sum_{k=1}^n (w_{ik}^t-w_{ik}^\star) (\xx_i^t - \xx_k^t)}_2^2 \\
        &+ 8\eta\rho c\norm*{ \frac{1}{2c} \sum_{k=1}^n (w_{jk}^t-w_{jk}^\star) (\xx_j^t - \xx_k^t) }_2^2
        + 4\eta\rho c\norm*{  \frac{1}{2\rho c}\left(\nabla f_i(\xx_i^t) - \nabla f_j(\xx_j^t)\right)}_2^2 \\
        \le& 
         (1-2\eta\rho c)\norm*{\xx_i^{t} - \xx_j^{t}}_2^2
         + \frac{2n\eta\rho}{c} \sum_{k=1}^n|w_{ik}^t-w_{ik}^\star|^2 \norm{\xx_i^t - \xx_k^t}_2^2\\
         &+ \frac{2n\eta\rho}{c} \sum_{k=1}^n|w_{jk}^t-w_{jk}^\star|^2 \norm{\xx_j^t - \xx_k^t}_2^2 
            + \frac{\eta}{\rho c} \underbrace{\norm{\nabla f_i(\xx_i^t) - \nabla f_j(\xx_j^t)}_2^2}_{=:\cT}.
    \end{align*}
    The last term $\cT$ can be upper bounded by adding $\pm \nabla f_i\left(\zz_{ij}^t\right)
        \pm \nabla f_j\left( \zz_{ij}^t \right)$ and use L-smoothness assumption \Cref{a:smoothness} of $f_i$ and that $i$, $j$ belong to the same cluster \Cref{a:collaborative}
    \begin{align*}
        \cT
        =&\norm*{\nabla f_i(\xx_i^t) \pm \nabla f_i\left(\zz_{ij}^t\right)
        \pm \nabla f_j\left( \zz_{ij}^t \right) - \nabla f_j(\xx_j^t)}_2^2 \\
        \le& 3 \norm*{\nabla f_i(\xx_i^t) - \nabla f_i\left(\zz_{ij}^t \right) }_2^2 + 3 \norm*{ \nabla f_i\left(\zz_{ij}^t \right) - \nabla f_j\left(\zz_{ij}^t\right)}_2^2 \\
        & +3 \norm*{\nabla f_j(\xx_j^t) - \nabla f_j\left(\zz_{ij}^t\right) }_2^2 \\
        \le& \frac{3L^2}{2} \norm{\xx_i^t - \xx_j^t}_2^2 + 3M^2_{ij} \norm*{ \nabla f_i\left(\zz_{ij}^t \right) + \nabla f_j\left(\zz_{ij}^t \right)}_2^2 \\
        =& \frac{3L^2}{2} \norm*{\xx_i^t - \xx_j^t}_2^2 + 3M^2_{ij} \norm*{\nabla \tilde{f}_{ij} (\zz_{ij}^t) }_2^2.
    \end{align*}
    By summing for all $i,j\in\cC$ 
    \begin{align*}
        \sum_{i,j\in\cC}  \norm*{\xx_i^{t+1} - \xx_j^{t+1}}_2^2 
        \le& (1-2\eta\rho c) \sum_{i,j\in\cC}\norm*{\xx_i^{t} - \xx_j^{t}}_2^2
        + 4n\eta\rho \sum_{i\in\cC} \sum_{k=1}^n|w_{ik}^t-w_{ik}^\star|^2 \norm{\xx_i^t - \xx_k^t}_2^2 \\
        &+ \frac{3\eta L^2}{2\rho c} \sum_{i,j\in\cC}\norm*{\xx_i^{t} - \xx_j^{t}}_2^2
        + \frac{3\eta M^2_{ij}}{\rho c} \sum_{i,j\in\cC} \norm*{\nabla \tilde{f}_{ij} (\zz_{ij}^t) }_2^2
    \end{align*}
    Use the previous \Cref{lemma:sd} to bound $\sum_{i,j\in\cC} \norm*{\nabla \tilde{f}_{ij} (\zz_{ij}^t) }_2^2$
    \begin{align*}
        \sum_{i,j\in\cC}  \norm*{\xx_i^{t+1} - \xx_j^{t+1}}_2^2 
        \le& (1-2\eta\rho c) \sum_{i,j\in\cC}\norm*{\xx_i^{t} - \xx_j^{t}}_2^2
        + 4n\eta\rho \sum_{i\in\cC} \sum_{k=1}^n|w_{ik}^t-w_{ik}^\star|^2 \norm{\xx_i^t - \xx_k^t}_2^2 \\
        &+ \frac{3\eta L^2}{2\rho c} \sum_{i,j\in\cC}\norm*{\xx_i^{t} - \xx_j^{t}}_2^2 \\
        &+ \frac{3\eta M^2_{ij}}{\rho c} \left(
        \frac{2}{\eta} \sum_{i,j\in\cC}\left( \tilde{f}_{ij}\left(\zz_{ij}^{t}\right) - \tilde{f}_{ij}\left(\zz_{ij}^{t+1}\right)  \right) \right. \\
        & \left. \qquad+  \left(\frac{3L^2}{4} + 3c^2\rho^2 + \frac{ L\eta\rho^2 (c-2) 2 \sigma^2}{b} \right) \sum_{i,j\in\cC} \norm*{\xx_i^t-\xx_j^t}_2^2
        \right) \\
        &+\frac{3\eta M^2_{ij}}{\rho c} \left(3nc\rho^2 \sum_{i\in\cC} \sum_{k=1}^n |w_{ik}^{t+1} - w_{ik}^\star|^2 \norm*{\xx_i^t - \xx_k^t}_2^2
        + \frac{c^2 L\eta \sigma^2}{2}
        \right).
    \end{align*}
    Rearrange the terms
    \begin{align*}
        &\sum_{i,j\in\cC}  \norm*{\xx_i^{t+1} - \xx_j^{t+1}}_2^2 \\
        \le& \left(1-2\eta\rho c+ \frac{3\eta L^2}{2\rho c}+ \frac{3\eta M^2_{ij}}{\rho c} \left(\frac{3L^2}{4} + 3c^2\rho^2 + \frac{ L\eta\rho^2 (c-2) 2 \sigma^2}{b} \right)\right) \sum_{i,j\in\cC}\norm*{\xx_i^{t} - \xx_j^{t}}_2^2 \\
        &+ \left(4n\eta\rho+\frac{3\eta M^2_{ij}}{\rho c}3nc\rho^2\right) \sum_{i\in\cC} \sum_{k=1}^n|w_{ik}^t-w_{ik}^\star|^2 \norm{\xx_i^t - \xx_k^t}_2^2 \\
        &+ \frac{3\eta M^2_{ij}}{\rho c} 
        \frac{2}{\eta} \sum_{i,j\in\cC}\left( \tilde{f}_{ij}\left(\zz_{ij}^{t}\right) - \tilde{f}_{ij}\left(\zz_{ij}^{t+1}\right)  \right)
        +  \frac{3\eta M^2_{ij}}{\rho c}  \frac{c^2 L\eta \sigma^2}{2}.
    \end{align*}
    By taking $b\ge \frac{2}{c^2}2L\eta(c-2)\sigma^2$ and $\rho\ge \frac{\sqrt{3}L}{c}$ and $M_{ij}\le\frac{1}{5}$, the following inequality hold true
    \begin{align*}
         \frac{3\eta M^2_{ij}}{\rho c} \left(\frac{3L^2}{4} + 3c^2\rho^2 + \frac{ L\eta\rho^2 (c-2) 2 \sigma^2}{b} \right)
         \le \frac{3\eta M^2_{ij}}{\rho c} \frac{15}{4} \rho^2 c^2
         \le \frac{45}{4}\rho c \eta M^2_{ij}
         \le \frac{1}{2} \eta\rho c.
    \end{align*}
    The upper bound of $\sum_{i,j\in\cC}  \norm*{\xx_i^{t+1} - \xx_j^{t+1}}_2^2$ can be simplied
    \begin{align*}
        \sum_{i,j\in\cC}  \norm*{\xx_i^{t+1} - \xx_j^{t+1}}_2^2 
        \le& \left(1-\eta\rho c\right) \sum_{i,j\in\cC}\norm*{\xx_i^{t} - \xx_j^{t}}_2^2 \\
        &+5n\eta\rho\sum_{i\in\cC} \sum_{k=1}^n|w_{ik}^t-w_{ik}^\star|^2 \norm{\xx_i^t - \xx_k^t}_2^2 \\
        &+ \frac{6 M^2_{ij}}{\rho c} 
        \sum_{i,j\in\cC}\left( \tilde{f}_{ij}\left(\zz_{ij}^{t}\right) - \tilde{f}_{ij}\left(\zz_{ij}^{t+1}\right)  \right) +  \frac{3\eta M^2_{ij}}{\rho c}  \frac{c^2 L\eta \sigma^2}{2}.
    \end{align*}

\end{proof}

\subsection{Proof of \Cref{theorem:main}}
\begin{proof}
    Given \Cref{lemma:consensus_distance} and average over time $t=0$ over $T-1$ and take expectation to all randomness throughout training
    \begin{align*}
        \frac{1}{T}\sum_{t=0}^{T-1}\sum_{i,j\in\cC}  \E\left[\norm*{\xx_i^{t+1} - \xx_j^{t+1}}_2^2 \right]
        \le& \left(1-\eta\rho c\right) \frac{1}{T}\sum_{t=0}^{T-1}\sum_{i,j\in\cC}\E\left[\norm*{\xx_i^{t} - \xx_j^{t}}_2^2 \right] \\
        &+5n\eta\rho \frac{1}{T}\sum_{t=0}^{T-1} \sum_{i\in\cC} \sum_{k=1}^n \E\left[|w_{ik}^t-w_{ik}^\star|^2 \norm{\xx_i^t - \xx_k^t}_2^2\right] \\
        &+ \frac{6 M^2_{ij}}{\rho c} \frac{1}{T}\sum_{t=0}^{T-1}
        \sum_{i,j\in\cC} \E\left[\left( \tilde{f}_{ij}\left(\zz_{ij}^{t}\right) \right]- \tilde{f}_{ij}\left(\zz_{ij}^{t+1}\right)  \right)
        +  \frac{3\eta M^2_{ij}}{\rho c}  \frac{c^2 L\eta \sigma^2}{2}.
    \end{align*}
    Rearrange $\frac{1}{T}\sum_{t=0}^{T-1}\sum_{i,j\in\cC}  \E\left[\norm*{\xx_i^{t+1} - \xx_j^{t+1}}_2^2\right] $ yields
    \begin{align*}
        \frac{1}{T}\sum_{t=0}^{T-1}\sum_{i,j\in\cC}  \E\left[\norm*{\xx_i^{t+1} - \xx_j^{t+1}}_2^2 \right]
        \le&\frac{5n}{c} \frac{1}{T}\sum_{t=0}^{T-1} \sum_{i\in\cC} \sum_{k=1}^n \E\left[|w_{ik}^t-w_{ik}^\star|^2 \norm{\xx_i^t - \xx_k^t}_2^2\right] \\
        &+ \frac{6 M^2_{ij}}{\eta\rho^2 c^2} \frac{1}{T}\sum_{t=0}^{T-1}
        \sum_{i,j\in\cC}\E\left[\left( \tilde{f}_{ij}\left(\zz_{ij}^{t}\right) - \tilde{f}_{ij}\left(\zz_{ij}^{t+1}\right)  \right)\right]
        + \frac{3 M^2_{ij}}{\rho^2 }  \frac{ L\eta \sigma^2}{2}.
    \end{align*}
    Consider bounding $|w_{ik}^t-w_{ik}^\star|^2$ in two cases

    \textbf{Case 1: $w_{ik}^\star=1$. } Suppose $M_{ik}\in(0,1)$, then $\norm{\nabla f_i(z^t_{ik}) - \nabla f_k(z^t_{ik})}_2^2 \le M^2_{ij} \norm{\nabla f_i(z^t_{ik}) + \nabla f_j(z^t_{ik})}_2^2$ implies 
    \begin{align*}
        \langle \nabla f_i(z^t_{ik}), \nabla f_k(z^t_{ik}) \rangle \ge \frac{1-M^2_{ik}}{2(1+M^2_{ik})} \left( \norm{\nabla f_i(z^t_{ik})}_2^2
        + \norm{\nabla f_k(z^t_{ik})}_2^2 \right) 
        \ge 0.
    \end{align*}
    then $w_{ik}^{t+1}= w_{ik}^\star = 1$ and therefore $|w_{ik}^{t+1} - w_{ik}^\star|^2=0$.

    \textbf{Case 2: $w_{ik}^\star=0$. } Suppose $\zeta^2_{ik}\ge \norm{ \nabla f_i(\xx) + \nabla f_k(\xx) }_2^2$ for all $\xx$ then
    \begin{align*}
        \norm{ \nabla f_i(\zz^t_{ik}) + \nabla f_k(\zz^t_{ik}) }_2^2
        =& \norm{ \nabla f_i(\zz^t_{ik})  }_2^2 + \norm{ \nabla f_k(\zz^t_{ik}) }_2^2
        + 2\langle \nabla f_i(\zz^t_{ik}), \nabla f_k(\zz^t_{ik}) \rangle \\
        \ge& \zeta^2_{ik} + 2\langle \nabla f_i(\zz^t_{ik}), \nabla f_k(\zz^t_{ik}) \rangle
    \end{align*}
    which means the inner product $\langle \nabla f_i(z^t_{ij}), \nabla f_j(z^t_{ij}) \rangle \le 0$ is negative, i.e., $w_{ij}^{t+1}=0=w_{ij}^\star$.

    Then with lower bound assumption of $f_i$ and $f_j$ \Cref{a:global_minimum}
    \begin{align*}
        \frac{1}{T}\sum_{t=0}^{T-1}\sum_{i,j\in\cC}  \E\left[\norm*{\xx_i^{t+1} - \xx_j^{t+1}}_2^2 \right]
        \le& \frac{6 M^2_{ij}}{\eta\rho^2 c^2} \frac{1}{T}\sum_{t=0}^{T-1}
        \sum_{i,j\in\cC} \E\left[\left( \tilde{f}_{ij}\left(\zz_{ij}^{t}\right) - \tilde{f}_{ij}\left(\zz_{ij}^{t+1}\right)  \right) \right]
        + \frac{3 M^2_{ij}}{\rho^2 }  \frac{ L\eta \sigma^2}{2}\\
        \le& \frac{6M^2_{ij}}{\eta\rho^2 c^2 T} 
        \sum_{i,j\in\cC}  \left( \tilde{f}_{ij}\left(\zz_{ij}^{0}\right) - \tilde{f}_{ij}^\star  \right)
        + \frac{3 M^2_{ij}}{\rho^2 }  \frac{ L\eta \sigma^2}{2}.
    \end{align*}
    Minimize the upper bound through choosing $\eta$ 
    \begin{align*}
        \eta\le \frac{2}{\sigma\sqrt{LT}} \sqrt{ \frac{1}{c^2} \sum_{i,j\in\cC}  \left( \tilde{f}_{ij}\left(\zz_{ij}^{0}\right) - \tilde{f}_{ij}^\star  \right)}
    \end{align*}
    such that
    \begin{align}\label{eq:xi_xj}
        \frac{1}{T}\sum_{t=0}^{T-1}\sum_{i,j\in\cC}  \E\left[\norm*{\xx_i^{t+1} - \xx_j^{t+1}}_2^2 \right]
        \le& \frac{6M^2_{ij}}{\rho^2 } 
         \sqrt{ \frac{L\sigma^2}{c^2T} \sum_{i,j\in\cC} \left( \tilde{f}_{ij}\left(\zz_{ij}^{0}\right) - \tilde{f}_{ij}^\star  \right)}.
    \end{align}
    By the result of \Cref{lemma:sd}
    \begin{align*}
       \frac{1}{T}\sum_{t=0}^{T-1} \sum_{i,j\in\cC} \E\left[\norm*{\nabla \tilde{f}_{ij}\left(\zz_{ij}^{t}\right)  }_2^2\right]
        \le& \frac{2}{\eta} \frac{1}{T} \sum_{i,j\in\cC}\left( \tilde{f}_{ij}\left(\zz_{ij}^{0}\right) - \tilde{f}_{ij}^\star  \right)
        + 4c^2\rho^2  \frac{1}{T}\sum_{t=0}^{T-1} \sum_{i,j\in\cC} \E\left[\norm*{\xx_i^t-\xx_j^t}_2^2 \right]
        + \frac{c^2 L\eta \sigma^2}{2}\\
        \le& 2c^2\sqrt{ \frac{L\sigma^2}{c^2T} \sum_{i,j\in\cC} \left( \tilde{f}_{ij}\left(\zz_{ij}^{0}\right) - \tilde{f}_{ij}^\star  \right)}
        + 4c^2\rho^2  \frac{1}{T}\sum_{t=0}^{T-1} \sum_{i,j\in\cC} \E\left[\norm*{\xx_i^t-\xx_j^t}_2^2\right] \\
        \le& \left(2c^2 + 24c^2M^2_{ij}\right) \sqrt{ \frac{L\sigma^2}{c^2T} \sum_{i,j\in\cC} \left( \tilde{f}_{ij}\left(\zz_{ij}^{0}\right) - \tilde{f}_{ij}^\star  \right)} \\
        \le& 3c^2 \sqrt{ \frac{L\sigma^2}{c^2T} \sum_{i,j\in\cC} \left( \tilde{f}_{ij}\left(\zz_{ij}^{0}\right) - \tilde{f}_{ij}^\star  \right)}.
    \end{align*}

\end{proof}

\subsection{Proof of \Cref{eq:corollary}}
\begin{proof}
    By adding $\norm*{\nabla f_i(x) + \nabla f_j(x) }_2^2$ on both sides of \eqref{eq:in_cluster}, and replace $x$ with $\zz_{ij}$ we have
    \begin{equation}
        2\left(\norm*{\nabla f_i(\zz_{ij})}_2^2 + \norm*{\nabla f_j(\zz_{ij})}_2^2\right) \le 4(1+M_{ij}^2)\norm*{\nabla \tilde{f}_{ij}(\zz_{ij})}_2^2 \qquad \forall~\xx\in\R^d.
    \end{equation}
    Then using the upper bound of $M_{ij} < 1/5$ from \Cref{theorem:main},  and average over $t$ and $i,j$ yields
    \begin{align}\label{eq:nabla_fi_z_ij}
       \frac{1}{c^2T}\sum_{t=0}^{T-1} \sum_{i,j\in\cC} \E\left[\norm*{\nabla f_i\left(\zz_{ij}^{t}\right)}_2^2\right]
        \le& \left(1+\frac{1}{25}\right)3 \sqrt{ \frac{L\sigma^2}{c^2T} \sum_{i,j\in\cC} \left( \tilde{f}_{ij}\left(\zz_{ij}^{0}\right) - \tilde{f}_{ij}^\star  \right)}.
    \end{align}
    By Cauchy-Schwarz inequality and $L$-Lipschitz smoothness, we have that
    \begin{align*}
         \frac{1}{c^2T}\sum_{t=0}^{T-1} \sum_{i,j\in\cC} \norm*{ \nabla f_i\left(\xx_{i}^{t}\right)  }_2^2 
        \le & \frac{1}{c^2T}\sum_{t=0}^{T-1} \sum_{i,j\in\cC} \norm*{\nabla f_i\left(\zz_{ij}^{t}\right)  }_2^2 
        +  \frac{1}{c^2T}\sum_{t=0}^{T-1} \sum_{i,j\in\cC} \norm*{\nabla f_i\left(\zz_{ij}^{t}\right) - \nabla f_i\left(\xx_{i}^{t}\right)  }_2^2 \\
        \le & \frac{1}{c^2T}\sum_{t=0}^{T-1} \sum_{i,j\in\cC} \norm*{\nabla f_i\left(\zz_{ij}^{t}\right)  }_2^2 
        + \frac{L^2}{4} \frac{1}{c^2T}\sum_{t=0}^{T-1} \sum_{i,j\in\cC} \norm*{ \xx_i^t - \xx_j^t }_2^2.
    \end{align*}
    Applying \eqref{eq:xi_xj} and  \eqref{eq:nabla_fi_z_ij} to the upper bound of the above inequality
    \begin{align*}
         \frac{1}{c^2T}\sum_{t=0}^{T-1} \sum_{i,j\in\cC} \norm*{ \nabla f_i\left(\xx_{i}^{t}\right)  }_2^2 
        \le & \left(1+\frac{1}{25}\right)3 \sqrt{ \frac{L\sigma^2}{c^2T} \sum_{i,j\in\cC} \left( \tilde{f}_{ij}\left(\zz_{ij}^{0}\right) - \tilde{f}_{ij}^\star  \right)} \\
        & 
        + \frac{L^2}{4} \frac{6M^2_{ij}}{\rho^2c^2} \sqrt{ \frac{L\sigma^2}{c^2T} \sum_{i,j\in\cC} \left( \tilde{f}_{ij}\left(\zz_{ij}^{0}\right) - \tilde{f}_{ij}^\star  \right)}.
    \end{align*}
    As $\rho c \ge \sqrt{3}L$ and $M_{ij} < \frac{1}{5}$ as stated in \Cref{theorem:main}, 
    \begin{align*}
         \frac{1}{c^2T}\sum_{t=0}^{T-1} \sum_{i,j\in\cC} \norm*{ \nabla f_i\left(\xx_{i}^{t}\right)  }_2^2 
        \le & 4 \sqrt{ \frac{L\sigma^2}{c^2T} \sum_{i,j\in\cC} \left( \tilde{f}_{ij}\left(\zz_{ij}^{0}\right) - \tilde{f}_{ij}^\star  \right)}.
    \end{align*}
\end{proof}    

\section{Experimental Details}
\label{sec:exp-details}
In Section \ref{sec:exps}, we present our results on two tasks with different properties. Here, we provide the full details of our experimental setup, alongside with additional experiments. 

We first describe the setup for cross-device and cross-silo experiments: we use the fix batch size of 128 for cross-device, and cross-silo experiments on CIFAR-100. We tune each method for the optimal learning rate individually: we use learning rate of 0.1 for ditto, 0.05 for Federated Clustering (FC), and 0.01 for all other methods. For Ditto, we use the hyper-parameter of $\lambda = 1$ as recommended in their paper. For Federated Clustering, we use the ground truth number of clusters and size of clusters as the hyper-parameter. We also use the ground truth number of clusters for IFCA, and sample all the clients in cross-silo experiment. We reduce the sampling rate to 10\% for the cross-device experiment to ensure scalability and fairness for comparison to other methods. For cross-silo experiments we employed a single NVIDIA V-100 GPU with 32GB memory, and moved to four NVIDIA V-100 GPUs with 32 GB memory for cross-device experiment. With this setup, running \ouralgo for cross-silo and cross-device experiment takes 9 hours and 28 hours respectively.

For Language modelling experiment, we conducted the experiments with the learning rate of 0.002, batch size of 50, and 4 accumulation steps. Note that each agent only get a subset of the regarding language from Wiki-40B dataset, consisting of total of 800000 tokens. We also used the context length of 512, dropout rate of 0.1, and LoRA module with rank 4. Training is performed on a single NVIDIA A-100 GPU with 40GB memory. It takes 2.5 hours to run \ouralgo for 500 iterations in this framework. We also use $\lambda = 0.1$ which has higher performance for this experiment.

\begin{figure}[ht!]
    \centering
    \includegraphics[width=0.9\textwidth]{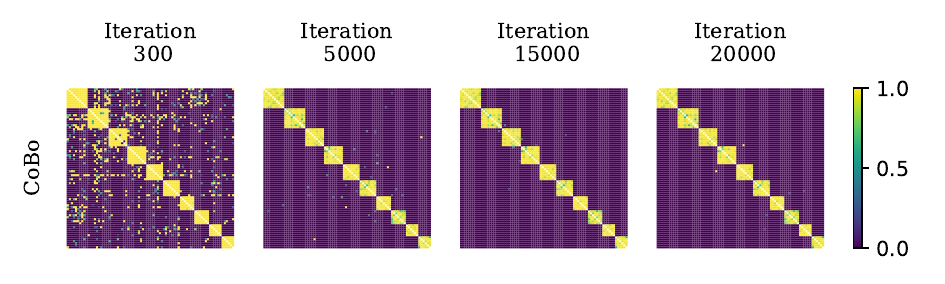} 
    \caption{
    Collaboration matrices learned by \ouralgo at different stages of training for cross-device experiment with 80 clients. The diagonals are masked out. The oracle matrix is a block diagonal matrix, consisting of 10 blocks: two blocks of size 10, two blocks of size 9, and so on. The collaboration matrix of \ouralgo already starts to look similar to oracle matrix within as low as 300 iterations. (1.5\% of the total iterations)}
    \label{fig:heatmaps80}
\end{figure}

\begin{figure}[ht!]
    \centering
    \includegraphics[width=0.9\textwidth]{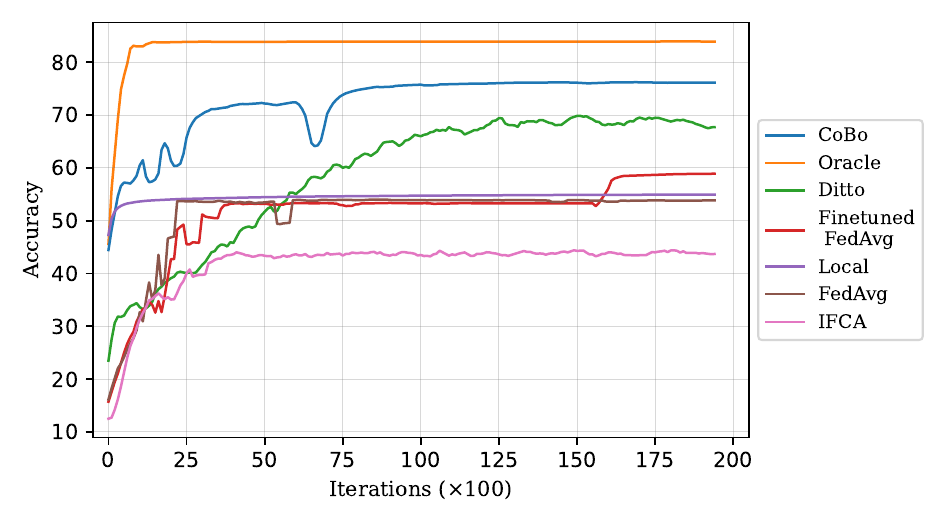} 
    \caption{Averaged accuracies of personalized models for cross-device federated learning with 80 clients.
    The "Oracle" denotes applying FedAvg to the clients having the data from the same classes of CIFAR-100 dataset.}
    \label{fig:accuracy-80}
\end{figure}

For the cross-device experiment with 80 agents in Section \ref{subsec:lang-exp}, we present the accuracy curve in Figure \ref{fig:accuracy-80}. Our method outperform all other methods except the Oracle with a large margin. We can also observe the collaboration matrix of \ouralgo in Figure \ref{fig:heatmaps80}. The clusters are learned with \ouralgo efficiently. 


\end{document}